%% file: sgd_shb_arxiv.tex
\documentclass[a4paper,11pt,twoside]{article}
\usepackage[left=2.5cm,right=2cm,top=2cm,bottom=2cm]{geometry}

\usepackage[T1]{fontenc}
\usepackage{lmodern}
\usepackage{url}            
\usepackage{booktabs}       
\usepackage{amsfonts}       
\usepackage{nicefrac}       
\usepackage{microtype}      
\usepackage{graphicx}
\usepackage{natbib}
\usepackage{subcaption}
\usepackage{tikz-cd}

\makeatletter
\newcommand{\printfnsymbol}[1]{%
  \textsuperscript{\@fnsymbol{#1}}%
}

 \usepackage{authblk}    
\usepackage{geometry}   
\usepackage{afterpage}

\usepackage[utf8]{inputenc} 
\usepackage[pagebackref]{hyperref}       
\hypersetup{
  colorlinks=true,
  linkcolor=blue,
  filecolor=magenta,
  urlcolor=cyan,
  citecolor=blue
}
\usepackage{url}            
\usepackage{booktabs}       
\usepackage{amsfonts}       
\usepackage{nicefrac}       
\usepackage{microtype}      
\usepackage{wrapfig}
\usepackage{amsmath,amsthm,amssymb,mathtools,graphicx,enumitem,hyphenat,float}
\usepackage{bbm}
\usepackage{subcaption}
\usepackage{accents}
\usepackage{import}
\graphicspath{{../figures/}}
\setcitestyle{citesep={,}}

\usepackage{appendix}

\input{header.tex}

\usepackage{cleveref}
\usepackage{autonum}
\usepackage{empheq}
\usepackage{booktabs}

\newcommand\Item[1][]{%
  \ifx\relax#1\relax  \item \else \item[#1] \fi
  \abovedisplayskip=0pt\abovedisplayshortskip=0pt~\vspace*{-\baselineskip}}

  \makeatletter
\newcommand*{\inlineequation}[2][]{%
  \begingroup
    \refstepcounter{equation}%
    \ifx\\#1\\%
    \else
      \label{#1}%
    \fi
    \relpenalty=10000 %
    \binoppenalty=10000 %
    \ensuremath{%
      #2%
    }%
    ~\@eqnnum
  \endgroup
}
\makeatother

\newtheorem{condition}{Condition}

\usepackage{thmtools}

\usepackage{times}

\begin{document}

\date{}

\title{\bf \emph{Almost sure} convergence rates for Stochastic Gradient Descent and Stochastic Heavy Ball}

\author[1]{Othmane Sebbouh}
\author[2]{Robert M. Gower}
\author[3]{Aaron Defazio}

\affil[1]{CNRS, ENS, PSL University, CREST, ENSAE}
\affil[2]{Télécom ParisTech, IP Paris}
\affil[3]{Facebook AI Research, New York, USA}

\maketitle

\begin{abstract}
	We study stochastic gradient descent (SGD) and the stochastic heavy ball method (SHB, otherwise known as the momentum method) for the general stochastic approximation problem. 
	For SGD, in the convex and smooth setting, we provide the first \emph{almost sure} asymptotic convergence \emph{rates} for a weighted average of the iterates . More precisely, we show that the convergence rate of the function values is arbitrarily close to $o(1/\sqrt{k})$, and is exactly $o(1/k)$ in the so-called overparametrized case. We show that these results still hold when using stochastic line search and stochastic Polyak stepsizes, thereby giving the first proof of convergence of these methods in the non-overparametrized regime.
 	Using a substantially different analysis, we show that these rates hold for SHB as well, but at the last iterate. This distinction is important because it is the last iterate of SGD and SHB which is used in practice. We also show that the last iterate of SHB converges to a minimizer \emph{almost surely}. Additionally, we prove that the function values of the deterministic HB converge at a $o(1/k)$ rate, which is faster than the previously known $O(1/k)$.
 	Finally, in the nonconvex setting, we prove similar rates on the lowest gradient norm along the trajectory of SGD.
\end{abstract}

\section{Introduction}
\label{sec:introduction} Consider the stochastic approximation problem
\begin{eqnarray} \label{eq:prob}
x_* \in \argmin_{x\in\R^d} f(x) \eqdef \ec[v \sim \D]{f_v(x)},
\end{eqnarray}
where $\D$ is a distribution on an arbitrary space $\Omega$ and $f_v$ is a real-valued function.  Let $\cX_* \subset \R^d$ be the set of solutions of~\eqref{eq:prob} (which we assume to be nonempty) and $f_* = f(x_*)$ for any solution $x_* \in \X_*$. The stochastic approximation problem \eqref{eq:prob} encompasses several problems in machine learning, including Online Learning and Empirical Risk Minimization (ERM). In these settings,  when the function $f$ can be accessed only through sampling or when the size of the datasets is very high, first-order stochastic gradient methods have proven to be very effective thanks to their low iteration complexity. The methods we analyze, Stochastic Gradient descent (SGD, \citep{RobbinsMonro:1951}) and Stochastic Heavy Ball (SHB, \citep{Polyak64}), are among the most popular such methods.

\subsection{Contributions and Background}\label{subsec:related_work}

Here we summarize the relevant background and our contributions. All of our rates of convergence are also given succinctly in Table~\ref{tab:summary_rates}.
\paragraph{\textit{Almost sure} convergence rates for SGD.} The \emph{almost sure} convergence of the iterates of SGD is a well-studied question \citep{Bottou03, Zhou17, Nguyen18}. For functions satisfying $\forall (x, x_*) \in \R^d \times \X_*, \langle \nabla f(x), x - x_* \rangle \geq 0$, called \textit{variationally coherent}, the convergence was shown in \cite{Bottou03}  by assuming that the minimizer is unique. Recently in~\cite{Zhou17}, the uniqueness assumption of the minimizer was dropped for variationally coherent functions by assuming bounded gradients. The easier question of the \emph{almost sure} convergence of the norm of the gradients of SGD in the nonconvex setting, and of the objective values in the convex setting, has also been positively answered by several works, see \cite{bertsekas2000gradient} and references therein, or more recently \cite{mertikopoulos2020almost, orabona2020almostsure}. In this work, we aim to quantify this
\citep{Nemirovski09, bach2011non, Ghadimi13} convergence. Indeed, while convergence rates are commonplace for convergence in expectation (\cite{Nemirovski09, bach2011non, Ghadimi13} for example), the litterature on the convergence rates of SGD in the \textit{almost sure} sense is sparse. For an adaptive SGD method, \cite{Li19} prove the convergence \emph{of a subsequence} of the squared gradient at a rate arbitrarily close to $o(1/\sqrt{k})$. More precisely, they show that $\liminf_k k^{\frac{1}{2} - \epsilon}\sqn{\nabla f(x_k)}=0$ for all $\epsilon	> 0$, where $x_k$ is the $k$th iterate of SGD. \cite{godichon2016lp} proves that, for locally strongly convex functions, the sequence $(\sqn{x_k - x_*})_k$, where $x_*$ is the unique minimizer of $f$, converges \emph{almost surely} at a rate arbitrarily close to $o(1/k)$.\\[0.1cm]

\noindent \emph{Contributions}:
{\bf 1.} In the convex and smooth setting, we show that the function values at a weighted average of the iterates of SGD converge \emph{almost surely} at a rate arbitrarily close to $o(1/\sqrt{k})$. In the so-called overparametrized case, where the stochastic gradients at any minimizer $\nabla f_v(x_*)$ are $0$, we show that this rate improves to $o(1/k)$. The proof of these results is surprisingly simple, and relies on a new weighted average of the iterates of SGD and on the classical Robbins-Siegmund supermartingale convergence theorem (Lemma~\ref{lem:simple_RS}). We also complement the well-known Robbins-Monro~\citep{RobbinsMonro:1951} conditions on the stepsizes with new conditions (See Condition~\ref{con:step_sizes}) that allow us to derive convergence rates in the \emph{almost sure} sense. We also show that our theory still holds in the nonsmooth setting when we assume bounded subgradients (Appendix \ref{sec:app_nonsmooth}). 
{\bf 2.}  In the nonconvex setting, under the recently introduced \emph{ABC condition}~\citep{Khaled20}, we derive \textit{almost sure} convergence rates for the minimum squared gradient norm along the trajectory of SGD which match the rates we derived for the objective values of SGD.

\paragraph{Asymptotic convergence of SGD with adaptive step sizes.} One drawback of the theory of SGD in the smooth setting is that it relies on the knowledge of the smoothness constant. Two of the earliest methods which have been proposed to address this issue are Line-Search (LS) \citep{nocedal2006sequential} and Polyak Stepsizes (PS) \citep{polyak1987introduction}. But while their convergence had been established in the deterministic case, it wasn't until recently \citep{Vaswani19, vaswani2020adaptive, loizou2020stochastic} that SGD with LS and with PS has been shown to converge assuming only smoothness and convexity of the functions $f_v$. For both methods, it has been shown that SGD converges to the minimum at a rate $O(1/k)$ in the overparametrized setting, but converges only to a neighborhood of the minimum when overparametrization does not hold.\\[0.1cm]

\noindent \emph{Contributions.} We show that SGD with LS or PS converges asymptotically at a rate arbitrarily close to $O(1/\sqrt{k})$ in expectation, and to $o(1/\sqrt{k})$ \emph{almost surely}. Moreover, in the overparamztrized setting, using the proof technique we developed for regular SGD, we show that SGD with LS or PS converges \emph{almost surely} to the minimum at a $o(1/k)$ rate.

\paragraph{\textit{Almost sure} convergence rates for SHB  and $o(1/k)$ convergence for HB.} The first local convergence of the deterministic Heavy Ball method was given in~\cite{Polyak64}, showing that it converges at an accelerated rate for twice differentiable strongly convex functions.  Only recently did \cite{Ghadimi2014}  show that the deterministic Heavy Ball method converged globally and sublinearly for smooth  and convex functions. The  SHB has recently been analysed for nonconvex functions and for strongly convex functions in~\cite{Gadat18}.  For strongly convex functions, they prove a $O\br{1/t^\beta}$ convergence rate for any $\beta < 1$.
Using a similar Lyapunov function to the one in \cite{Ghadimi2014}, a $O(1/\sqrt{t})$ convergence rate for SHB in the convex setting was given in~\cite{Yang16} and~\cite{Orvieto19} under the bounded gradient variance assumption. For the specialized setting of minimizing quadratics, it has been shown that the SHB iterates converge  linearly at an accelerated rate, but only in expectation rather than in L2~\citep{Loizou2018}. By using stronger assumptions on the noise as compared to \cite{Kidambi18}, \cite{Can19} show that by using a specific parameter setting, the SHB applied on quadratics converges at an accelerated rate to a neighborhood of a minimizer. Finally, the \textit{almost sure }convergence of SHB to a minimizer for nonconvex functions was proven in \cite{Gadat18} under an \emph{elliptic condition} which guarantees that SHB escapes any unstable point. But we are not aware of any convergence rates for the \textit{almost sure} convergence of SHB.\\[0.1cm]

\noindent \emph{Contributions.} 
{\bf 1.} In the smooth and convex setting, we show that the function values at the last iterate of SHB converge \emph{almost surely} at a rate close to $o(1/\sqrt{k})$. Similarly to SGD, this rate can be improved to $o(1/k)$ in the overparametrized setting. Moreover, we show that the last iterate of SHB converges to a minimizer \emph{almost surely}. In the deterministic setting, where we use the gradient $\nabla f$ at each iteration,  we prove that the function values of the deterministic HB converge at a $o(1/k)$ rate, which is faster than the previously known $O(1/k)$ \citep{Ghadimi2014} and matches the rate recently derived for Gradient Descent in \cite{Lee19}. Compared to the SGD analysis we develop, the derivation of \textit{almost sure} convergence rates for SHB is quite involved, and combines tools developed in \cite{Attouch16} for the analysis of the (deterministic) Nesterov Accelerated Gradient method and the classical Robbins-Siegmund theorem.
{\bf 2.} Our results rely on an iterate averaging viewpoint of SHB (Proposition \ref{prop:ima}), which considerably simplifies our analysis and suggests parameter settings different from the usual settings of the momentum parameter, which is fixed at around 0.9, and often exhibits better empirical performance than SGD \citep{Sutskever13}. We show through extensive numerical experiments in Figure~\ref{fig:experiments} that our new parameter setting is statistically superior to the standard rule-of-thumb settings on convex problems.
{\bf 3.} Additionally, we show in Appendix \ref{sec:app_shb_exp} that the bounded gradients and bounded noise assumptions used in \cite{Yang16, Orvieto19} can be avoided, and prove that SHB at the last iterate converges in expectation at a $O(1/k)$ rate to a neighborhood of the minimum and at a $O(1/\sqrt{k})$ rate to the minimum exactly.\\[-0.7cm]

\subsection{Assumptions and general consequences}
Our theory in the convex setting relies on the following assumption of convexity and smoothness.
\begin{assumption}\label{asm:smoothconvex}
For all $v \sim \D $, there exists $L_v>0$ such that for every $x, y \in \R^d$ we have that
\begin{align}\label{eq:convexity}
f_v(y) & \geq f_v(x) + \langle \nabla f_v(x), y - x \rangle , \\
f_v(y) &\leq f_v(x) + \langle \nabla f_v(x), y - x \rangle + \frac{L_v}{2}\sqn{y - x}, \label{eq:smoothn}
\end{align}
\textit{almost surely}. Let $\cL \eqdef \sup_{v \sim \D} L_v.$ We assume that $\cL < \infty$.
Consequently, $f$ is also smooth and we use $L>0$ to denote its smoothness constant.
\end{assumption}

\begin{definition}
Define the residual gradient noise as 
\begin{equation}\label{eq:gradnoise}
 \sigma ^2 \eqdef \sup_{x \in \X^*}\ec[v \sim \D]{\norm{\nabla f_v(x^*)}^2}.
\end{equation}
\end{definition}

Assumption~\ref{asm:smoothconvex} has the following simple consequence on the expectation of the gradients.
\begin{lemma}
\label{lem:smoothconvex}
If Assumption~\ref{asm:smoothconvex} holds, then 
\begin{equation}\label{eq:expsmooth}
\ecd{\sqn{\nabla f_v(x)}} \leq 4\cL\br{f(x) - f_*} + 2\sigma^2.
\end{equation}
\end{lemma}
In all our results of Sections \ref{sec:sgd_asymp} and \ref{sec:shb_asymptotic}, we only use convexity and the  inequality~\eqref{eq:expsmooth}. Thus, Assumption \ref{asm:smoothconvex} can be slightly relaxed by removing the smoothness condition~\eqref{eq:smoothn} and re-branding~\eqref{eq:expsmooth} as an assumption, as opposed to a consequence. 
  With this bound~\eqref{eq:expsmooth}, we do not need to assume a uniform bound on the the squared norm of the gradients or on their variance, as is often done when analyzing SGD~\citep{Nemirovski09} or SHB~\citep{Yang16}. Note, however, that the analysis carried for SGD and SHB in~\cite{Nemirovski09} and \cite{Yang16} is more general and applies to the nonsmooth case, for which assuming bounded subgradients is often necessary. As an illustration, we show that our results hold in the nonsmooth case under the bounded subgradients assumption in Appendix \ref{sec:app_nonsmooth}. Note also that all our results still hold with the usual but more restrictive assumption of bounded gradient variance (see for example \cite{Ghadimi13}). Indeed, when this assumption holds, \eqref{eq:expsmooth} holds with $L$ in place of $\cL$, where $L$ is the smoothness constant of $f$.

\begin{definition}[Informal]
When $\sigma^2 = 0$, we say that we have an overparametrized model.
\end{definition}
When our models have enough parameters to \emph{interpolate the data}~\citep{Vaswani18}, then $\nabla f_v(x^*) =0, \; \forall v \sim \D$, and consequently $\sigma^2 =0.$ This property has been observed especially for the training of large neural networks in Empirical Risk Minimization, where $f$ is a finite-sum.

\begin{remark}[Finite-sum setting]\label{rem:finite_sum}
Let $n \in \N^*$ and define $[n] \eqdef \left\{1,\dots,n\right\}$. Let $f(x) = \frac{1}{n}\sum_{i=1}^n f_i(x)$, where each $f_i$ is $L_{i}$-smooth and convex, and $L_{\max} = \max_{i \in [n]}L_i < \infty$. If we sample minibatches of size $b$ without replacement, then \cite{SAGAminib, Gower19} show that \eqref{eq:expsmooth} holds with\\[-0.7cm]
\begin{align}
\cL \equiv \cL(b) \eqdef  \frac{1}{b}\frac{n-b}{n-1}L_{\max} + \frac{n}{b}\frac{b-1}{n-1}L \quad \mbox{and} \quad \sigma^2 \equiv \sigma^2(b) = \frac{1}{b}\frac{n - b}{n - 1}\sigma_1^2,
\end{align}
where $\sigma_1^2\eqdef  \frac{1}{n}\underset{x \in \cX_*}{\sup}\sum_{i=1}^n \sqn{\nabla f_i(x_*)}$. Note that $\sigma^2(n)~=~0$ and $\cL(n)=L$, as expected, since $b=n$ corresponds to full batch gradients,  or equivalently to using deterministic GD or HB.
Similarly, $\cL(1) = L_{\max}$, since $b=1$ corresponds to sampling one individual $f_i$ function. 
\end{remark}

\subsection{SGD and an iterate-averaging viewpoint of SHB}
In Section \ref{sec:sgd_asymp}, we will analyze SGD, where we sample at each iteration $v_k \sim \D$, and iterate
\begin{align}\label{eq:SGD}
\boxed{x_{k+1} = x_k - \eta_k \nabla f_{v_k}(x_k),} \tag{SGD}
\end{align}
where $\eta_k$ is a step size. In Section \ref{sec:shb_asymptotic}, we will analyze SHB, whose iterates are
\begin{align}\label{eq:SHB}
x_{k+1} = x_k - \alpha_k \nabla f_{v_k}(x_k) + \beta_k \br{x_k - x_{k-1}}, \tag{SHB}
\end{align}
where $\alpha_k$ is commonly referred to as the step size and $\beta_k$ as the momentum parameter. 
Our forthcoming analysis of~\eqref{eq:SHB} leverages an \emph{iterate moving-average viewpoint} of~\eqref{eq:SHB} and particular parameter choices 
that we present in Proposition \ref{prop:ima}.
\begin{proposition}
\label{prop:ima} Let $z_0 = x_0 \in \R^d$ and $\eta_k,\lambda_k > 0.$ Consider the iterate-moving-average (IMA) method:
\begin{align}
\boxed{z_{k+1} = z_{k}-\eta_{k}\nabla f_{v_k}(x_{k}), \quad x_{k+1} = \frac{\lambda_{k+1}}{\lambda_{k+1}+1}x_{k}+\frac{1}{\lambda_{k+1}+1}z_{k+1}} \label{eq:SHB_IMA} \tag{SHB-IMA}
\end{align}\\[-0.8cm]
\begin{align}
\mbox{If} \qquad \quad \alpha_k =  \frac{\eta_{k}}{1+\lambda_{k+1}}  \; \mbox{ and }\; \beta_k =  \frac{\lambda_{k}}{1+\lambda_{k+1}},\label{eq:alphafrometalambdaintro}
\end{align}
then the $x_k$ iterates in~\eqref{eq:SHB_IMA} are equal to the $x_k$ iterates of the  method~\eqref{eq:SHB} .
\end{proposition}
The equivalence between this formulation and the original \ref{eq:SHB} is proven in the supplementary material (Section \ref{sec:proof_ima}). The IMA formulation \eqref{eq:SHB_IMA} is crucial in comparing \ref{eq:SHB} and \ref{eq:SGD} as it allows to interpret the parameter $\alpha_k$ in \ref{eq:SHB} as a \emph{scaled} step size and unveils a natural stepsize $\eta_k$. In all of our theorems, the parameters $\eta_k$ and $\lambda_k$ naturally arise in the recurrences and Lyaponuv functions. We determine how to set the parameters $\eta_k$ and $\lambda_k,$ which in turn gives settings for $\alpha_k$ and $\beta_k$ through~\eqref{eq:alphafrometalambdaintro}.  In the remainder of this work, we will directly analyze the method \ref{eq:SHB_IMA}.

Having new reformulations often leads to new insights.
This is the case for Nesterov's accelerated gradient method, where at least six forms are known \citep{adefazio-curvedgeom2019} and recent research suggests that iterate-averaged reformulations are the easiest to generalize to the combined proximal and variance-reduced case \citep{lan2017}. \\[-1.5cm]
\begin{center}
\begin{table}
\centering
\addtolength{\leftskip} {-2cm} 
\addtolength{\rightskip}{-2cm}
\begin{tabular}{@{}lllllll@{}} \toprule
\textbf{Algorithm} & $\mathbf{\sigma^2}$ & \textbf{Stepsize} & \textbf{Conv.} & \textbf{Rate}   & \textbf{Iterate} & \textbf{Ref} \\ \midrule
\ref{eq:SGD}  & $\neq 0$ & $O(k^{-1/2 - \epsilon})$ & \emph{a.s.} & $o(k^{-1/2 + \epsilon})$  & average  & Cor. \ref{cor:stepsizes_rates_sgd} \\
\ref{eq:SGD}  & $=0$ & $O(1)$ & \emph{a.s.} & $o(k^{-1})$  & average & Cor. \ref{cor:stepsizes_rates_sgd} \\
\ref{eq:SGD-ALS}, \ref{eq:SGD-PS}  & $\neq0$ & adaptive$^*$ & \emph{a.s.}, $\mathbb{E}$ & $o(k^{-1/2 + \epsilon}), O(k^{-1/2 +\epsilon})$ & average & Cor. \ref{cor:stepsizes_rates_sgd_adap}   \\
\ref{eq:SGD-ALS}, \ref{eq:SGD-PS}  & $=0$ & adaptive$^*$ & \emph{a.s.}, $\mathbb{E}$ & $o(k^{-1}), O(k^{-1})$ & average & Cor. \ref{cor:stepsizes_rates_sgd_adap}   \\
\ref{eq:SHB}  & $\neq 0$ & $O(k^{-1/2 + \epsilon})$, $O(k^{-1/2})$ & \emph{a.s.}, $\mathbb{E}$ & $o(k^{-1/2 - \epsilon})$, $O(k^{-1/2})$ & last & Cor. \ref{cor:stepsizes_rates_shb}, \ref{cor:convneigh}  \\
\ref{eq:SHB}  & $=0$ & $O(1)$ & \emph{a.s.}, $\mathbb{E}$ & $o(k^{-1})$, $O(k^{-1})$  & last & Cor. \ref{cor:faster1k}, \ref{cor:convneigh}\\
\ref{eq:SGD}, nonconvex  & $C \neq 0$ & $O(k^{-1/2 + \epsilon})$ & \emph{a.s.} & $o(k^{-1/2 + \epsilon})$ & $\min \sqn{\nabla}$ & Cor. \ref{cor:stepsizes_rates_sgd_nonconvex}  \\
\ref{eq:SGD}, nonconvex  & $C =0$ & $O(1)$ & \emph{a.s.} & $o(k^{-1})$  & $\min \sqn{\nabla}$ & Cor. \ref{cor:stepsizes_rates_sgd_nonconvex} \\ \bottomrule
\end{tabular}
\captionsetup{width=1\linewidth}
\caption{Summary of the rates we obtain. 
All small-o (resp. big-O) rates are \emph{almost surely} (resp. in expectation). The constants $C, \cL$ and $\sigma$ are defined in~\eqref{eq:ABC},~\eqref{eq:expsmooth} and~\eqref{eq:gradnoise}, respectively. \textit{a.s.}: \textit{almost surely}, $\mathbb{E}$: in expectation. $\min \sqn{\nabla}$: lowest squared norm of the gradient along the trajectory of \ref{eq:SGD}. adaptive$^*$: Maximum step sizes need to verify conditions similar to Condition \ref{con:step_sizes}, but do not require knowing the smoothness constant $\cL$.}
\label{tab:summary_rates}
\end{table}
\end{center}

\section{\textit{Almost sure} convergence rates for SGD and SGD with adaptive stepsizes}\label{sec:sgd_asymp}
  We will first present \textit{almost sure} convergence rates for SGD, then for SGD with Line-Search and Polyak Stepsizes.
  \subsection{SGD: average-iterates \textit{almost sure }convergence}



Our results rely on a classical convergence result \citep{Robbins71}.
\begin{lemma}\label{lem:simple_RS}
Consider a filtration $\br{\F_k}_k$, the  nonnegative sequences of $\br{\F_k}_k-$adapted processes $\br{V_k}_k$,  $\br{U_k}_k$ and $\br{Z_k}_k$, and a sequence of positive numbers $\br{\gamma_k}_k$ such that $\sum_k Z_k < \infty \; \mbox{almost surely}$, $\prod_{k=0}^{\infty} (1+\gamma_k) < \infty$, and \[\forall k \in \N, \; \ec{V_{k+1}|\F_k} + U_{k+1} \leq (1+\gamma_k) V_k + Z_k.\]
Then $\br{V_k}_k$ converges and $\sum_k U_k < \infty $ \textit{almost surely}.
\end{lemma}

We use the following condition on the step sizes in our \textit{almost sure }convergence results.
\begin{condition}\label{con:step_sizes}
The sequence $\br{\eta_k}_k$ is decreasing, $\sum_k \eta_k = \infty$, $\sum_k \eta_k^2\sigma^2 < \infty$ and $\sum_k \frac{\eta_k}{\sum_{j}\eta_j} = \infty$.
\end{condition}
The conditions $\sum_k \eta_k = \infty$ and $\sum_k \eta_k^2 < \infty$ are known as the Robbins-Monro conditions \citep{RobbinsMonro:1951} and are classical in the SGD litterature (see \cite{bertsekas2000gradient} for example). The additional conditions, $\sum_k \frac{\eta_k}{\sum_{j}\eta_j} = \infty$ and $\br{\eta_k}_k$ is decreasing, allow us to derive convergence rates for the \textit{almost sure} convergence using a new proof technique. However, as we will see in the next remark, the usual choices of step sizes which verify the Robbins-Monro conditions verify Condition \ref{con:step_sizes} as well.

\begin{remark}\label{rem:stepsizes_condition}
Let $\eta_k = \frac{\eta}{k^{\xi}}$ with $\xi, \eta > 0$. Condition \ref{con:step_sizes} is verified for all $\xi \in (\frac{1}{2}, 1]$ when $\sigma^2 \neq 0$, and for all $\xi \in [0, 1]$ when $\sigma^2 = 0$.
\end{remark}
See Appendix \ref{sec:app_formal_statements} for a proof of this remark. Indeed, all the formal proofs of our results are defered to the appendix.

\begin{theorem}\label{thm:as_conv_sgd}
Let Assumption \ref{asm:smoothconvex} hold. Consider the iterates of \ref{eq:SGD}. Choose step sizes $\br{\eta_k}_k$ which verify Condition \ref{con:step_sizes}, where $\forall k \in \N, \; 0 < \eta_k \leq 1/(4\cL)$. Define for all $k \in \N$
\begin{eqnarray}\label{eq:w_xbar}
w_k = \frac{2\eta_k}{\sum_{j=0}^{k}\eta_j} \quad \mbox{and} \quad \left\{
            \begin{array}{l}
            		\bar{x}_0 = x_0 \\
            		\bar{x}_{k+1} = w_k x_k + (1 - w_k) \bar{x}_k.
            \end{array} \right. 
\end{eqnarray}
Then, we have $a.s$. that $f(\bar{x}_k) - f_* = o\br{\frac{1}{\sum_{t=0}^{k-1} \eta_t}}.$
\end{theorem}

\begin{proof}
We present the main elements of the proof which help in understanding the difference between the classical non-asymptotic analysis of SGD in expectation and our analysis. We present the complete proof in Section \ref{sec:app_proofs_as_conv_sgd} of the appendix.

In the convex setting, the bulk of the convergence proofs of SGD is in using convexity and smoothness of $f$ to establish that, if $\eta_k \leq \frac{1}{4\cL}$, we have
\begin{equation}\label{eq:sgd_rec_main}
\ec[k]{\sqn{x_{k+1} - x_*}} + \eta_k \br{f(x_k) - f_*} \leq \sqn{x_{k} - x_*} + 2\eta_k^2\sigma^2.
\end{equation}
\paragraph{Classic non-asymptotic convergence analysis for SGD.} Taking the expectation, using telescopic cancellation and Jensen's inequality, it is possible to establish that
\begin{equation}
\ec{f(\tilde{x}_k) - f_*} \leq \frac{\sqn{x_0 - x_*}}{\sum_{t=0}^{k-1}\eta_t} + \frac{2\sigma^2 \sum_{t=0}^{k-1}\eta_t^2}{\sum_{t=0}^{k-1}\eta_t}, \; \mbox{where } \; \tilde{x}_k = \sum_{t=0}^{k-1}\frac{\eta_t}{\sum_{j=0}^{k-1}\eta_j} x_t.
\end{equation}
$\tilde{x}_k$ can then be computed on the fly using: 
\begin{equation}
\tilde{x}_{k+1} = \tilde{w}_k x_k + (1 - \tilde{w}_k) \tilde{x}_k, \quad \mbox{where} \quad \tilde{w}_k = \frac{\eta_k}{\sum_{j=0}^{k}\eta_j}.  \label{eq:xbar_classic}
\end{equation}
 This sequence of weights $\br{\tilde{w}_k}_k$ (which can be computed on the fly as $\tilde{w}_{k+1} = \frac{\eta_{k+1}\tilde{w}_k}{\eta_k + \eta_{k+1}\tilde{w}_k}$) is the one which allows to derive the tightest upper bound on the objective gap $f(x) - f_*$ in expectation. But it does not lend itself to tight \textit{almost sure} asymptotic convergence, as we will show next.

\paragraph{Naive asymptotic analysis.}
Applying  Lemma \ref{lem:simple_RS} to \eqref{eq:sgd_rec_main} gives
that $\sum_k \eta_k\br{f(x_k) - f_*} < \infty.$ Unfortunately, this only gives that $\lim_k \eta_k \br{f(x_k) - f_*} = 0$. 

\paragraph{Asymptotic analysis using the iterates defined in \eqref{eq:xbar_classic}.}
 \textit{What if we had used the sequence of iterates defined in} \eqref{eq:xbar_classic}? Let $\tilde{\delta}_k = f(\tilde{x}_k) - f_*$. Using Jensen's inequality, we have
\[ f(x_k) - f_* \geq \frac{1}{\tilde{w}_k}\tilde{\delta}_{k+1}  - \br{\frac{1}{\tilde{w}_k}  - 1} \tilde{\delta}_k. \]
Using this bound in \eqref{eq:sgd_rec_main} gives, after replacing $\tilde{w}_k$ by its expression \eqref{eq:xbar_classic} and multiplying by $\eta_k$, that
\[\ec[k]{\sqn{x_{k+1} - x_*}} + \sum_{j=0}^k \eta_j \tilde{\delta}_{k+1} \leq \sqn{x_{k} - x_*} + \sum_{j=0}^{k-1} \eta_j \tilde{\delta}_k + 2\eta_k^2\sigma^2.\]
 
Applying Lemma \ref{lem:simple_RS} gives that $\br{\sum_{j=0}^{k-1}\eta_j \tilde{\delta}_k}_k$ converges \emph{almost surely}. Hence, there exist $k_0 \in \N$ and a constant $C_{k_0}$ such that for all $k \geq k_0, \; \tilde{\delta}_k \leq \frac{C_{k_0}}{\sum_{j=0}^{k-1}\eta_j}$. That is, we have \[\tilde{\delta}_k = O\br{\frac{1}{\sum_{j=0}^{k-1}\eta_j}}.\]
But we show that we can actually do much better.

\paragraph{Our analysis.}
Now consider the alternative averaging of iterates $\bar{x}_k$ given in~\eqref{eq:w_xbar}. First note that using \eqref{eq:sgd_rec_main} and Lemma \ref{lem:simple_RS}, we have that $\br{\sqn{x_k - x_*}}_k$ converges almost surely. Let $\delta_k \eqdef f(\bar{x}_k) - f_*$. 
As we have done in the last paragraph, we can use Jensen's inequality to lower-bound $f(x_k) - f_*$ in \eqref{eq:sgd_rec_main} (detailed derivations are given in Appendix \ref{sec:app_proofs_as_conv_sgd}), and we obtain:
\begin{equation}
\ec[k]{\sqn{x_{k+1} - x_*}} +  \frac{1}{2}\sum_{j=0}^{k}\eta_j \delta_{k+1}+ \frac{\eta_k}{2}\delta_k \leq \sqn{x_{k} - x_*} + \frac{1}{2}\sum_{j=0}^{k-1}\eta_j\delta_k
+  2\eta_k^2\sigma^2.
\end{equation}

By Lemma \ref{lem:simple_RS}, $\br{\sum_{j=0}^{k-1}\eta_j\delta_k}_k$ converges \emph{almost surely}, and~${\sum_k \eta_k\delta_k < \infty}$, which implies that ${\lim_{k} \eta_k \delta_k= 0.}$  But since $\sum_k \frac{\eta_k}{\sum_{j=0}^{k-1} \eta_j} = \infty$, we have the desired result: $\lim_k \sum_{j=0}^{k-1}\eta_j\delta_k= 0$.

Note that in the first iteration, $w_0 = 2$ and $\bar{x}_1 = x_0$, and we don't use Jensen's inequality.
\end{proof}

With suitable choices of stepsizes, we can extract \emph{almost sure} convergence rates for SGD, as we see in the next corollary. These choices and all the rates we derive are also summarized in Table \ref{tab:summary_rates}. To the best of our knowledge, these are the first rates for the \emph{almost sure} convergence of SGD in the convex setting.
\begin{corollary}[Corollary of Theorem \ref{thm:as_conv_sgd}]\label{cor:stepsizes_rates_sgd}
Let Assumption \ref{asm:smoothconvex} hold. Let $0 < \eta \leq 1/4\cL$ and $\epsilon > 0$. 
\begin{itemize}
\item if $\sigma^2 \neq 0$.  Let $\eta_k = \frac{\eta}{k^{1/2 + \epsilon}}$.\\[-0.7cm]
\begin{align}
f(\bar{x}_k) - f_* = o\br{\frac{1}{k^{1/2 - \epsilon}}}.
\end{align}\\[-1.5cm]
\item If $\sigma^2 =  0$.  Let $\eta_k = \eta $. Then\\[-0.7cm]
\begin{align}
f(\bar{x}_k) - f_* = o\br{\frac{1}{k}}.
\end{align}\\[-1.5cm]
\end{itemize}
\end{corollary}

Although the \textit{almost sure} convergence of SGD with favourable convergence rates only requires the step sizes to verify Condition \ref{con:step_sizes}, there are other popular methods to set the step sizes, such as Line-Search \citep{nocedal2006sequential} or Polyak Stepsizes \citep{polyak1987introduction}, which do not require knowing the smoothness constant $\cL$. A natural question is whether the result we have derived in Theorem \ref{thm:as_conv_sgd} extends to these methods. We answer this question positively in the next section.

\subsection{Convergence of Adaptive step size methods}
We first present two adaptive step size selection methods and then present their convergence analysis.\\[-0.2cm]

\noindent\begin{minipage}[t]{0.42\textwidth}
\paragraph{Armijo Line-Search Stepsize (ALS).} We say that $\alpha$ is an Armijo line-seach stepsize at $x \in \R^d$ for the function $g$ if, given constants $c, \alpha_{\max} > 0$, $\alpha$ is the \textit{largest step size} in $(0, \alpha_{\max}]$ such that
\begin{equation}\label{eq:armijo}
g(x - \alpha \nabla g(x)) \leq g(x) - c \alpha \norm{\nabla g(x)}_2^2,
\end{equation}
which we denote by $$\boxed{\alpha \sim \mathbf{ALS}_{c, \alpha_{\max}}(g,x)}$$
In practice,  we use backtracking to find this $\alpha$, where we start with a value $\alpha_{\max}$ and decrease it by a factor $\beta \in (0, 1)$ until \eqref{eq:armijo} is verified.
    \end{minipage}%
    \begin{minipage}{0.03\textwidth}\centering
     \,
    \end{minipage}%
    \begin{minipage}[t]{0.55\textwidth}
\paragraph{Polyak Stepsize (PS).} Let $g$ be a function lower bounded by $g^*$.
We say that $\alpha$ is a Polyak step size at $x \in \R^d$ if, given constants $c, \alpha_{\max} > 0$,
\begin{equation}\label{eq:sps}
\alpha = \min \left\{\frac{g(x) - g^*}{c\sqn{\nabla g(x)}}, \, \alpha_{max}\right\},
\end{equation}
which we denote by $$\boxed{\alpha \sim \mathbf{PS}_{c, \alpha_{\max}}(g,x)}$$
The drawback of this method is that we need to know $g^*$. There is a range of applications where we know this value and Polyak Stepsizes have been shown to work well experimentally. See \cite{loizou2020stochastic} for more details.
    \end{minipage}
    
\vskip.3\baselineskip
Instead of using a pre-determined step size in \ref{eq:SGD}, we can choose at each iteration $\eta_k \sim \mathbf{ALS}_{c, \alpha_{\max}}(f_{v_k}, x_k)$ or $\eta_k \sim\mathbf{PS}_{c, \alpha_{\max}}(f_{v_k}, x_k)$. SGD with ALS or PS is known to converge sublinearly to a neighborhood of the minimum and to the minimum exactly if $\sigma^2=0$ \citep{Vaswani19, vaswani2020adaptive, loizou2020stochastic}. However, it is still not known whether these methods converge to the minimum when $\sigma^2 \neq 0$.

Let $\br{\eta_k^{\max}}_k$ and $\br{\gamma_k}$ be two strictly positive decreasing sequences. Consider the following modified SGD methods: at each iteration $k$, sample $v_k \sim \D$ and update
\begin{align}
	x_{k+1} &= x_k - \eta_k \gamma_k \nabla f_{v_k}(x_k), \quad \mbox{where} \quad \eta_k \sim \mathbf{ALS}_{c, \eta_k^{\max}}(f_{v_k},x_k) \tag{SGD-ALS}, \label{eq:SGD-ALS}\\
	x_{k+1} &= x_k - \eta_k \gamma_k \nabla f_{v_k}(x_k), \quad \mbox{where} \quad \eta_k \sim \mathbf{PS}_{c, \eta_k^{\max}}(f_{v_k},x_k) \tag{SGD-PS}.\label{eq:SGD-PS}
\end{align}

 

\begin{assumption}\label{asm:lower-bound}
For all $v \sim \D, \; f_v$ is lower bounded by $f_v^* > -\infty$ \textit{almost surely}, and we define $\bar{\sigma}^2 \eqdef f_* - \ec[v]{f_{v}^*}$.
\end{assumption}

Similar to our analysis of \ref{eq:SGD}, we can derive \textit{almost sure} convergence rates to the minimum for an average of the iterates. Remarkably, the analysis of the two methods \ref{eq:SGD-ALS} and \ref{eq:SGD-PS} can be unified. 

\begin{theorem}\label{thm:adap_asymp}
Let Assumptions \ref{asm:smoothconvex} and \ref{asm:lower-bound} hold. Consider the iterates of \ref{eq:SGD-ALS} and \ref{eq:SGD-PS}. Choose $\br{\eta_k^{\max}}_k$ and $\br{\gamma_k}_k$ such that $\br{\eta_k^{\max}\gamma_k}_k$ is decreasing, $\eta_k^{\max} \rightarrow 0$, $\sum_k \eta_k^{\max}\gamma_k = \infty$, $\sum_k \eta_k^{\max}\gamma_k^2\sigma^2 < \infty$ and $\sum_k \frac{\eta_k^{\max}\gamma_k}{\sum_{j=0}^{k-1}\eta_j^{\max}\gamma_j} = \infty$, $c \geq \frac{1}{2}$ and $\gamma_k \leq c$. Define for all $k \in \N$
\begin{eqnarray}\label{eq:w_xbar_adap}
w_k = \frac{2\eta_k^{\max}\gamma_k}{\sum_{j=0}^{k}\eta_j^{\max}\gamma_j} \quad \mbox{and} \quad \left\{
            \begin{array}{l}
            		\bar{x}_0 = x_0 \\
            		\bar{x}_{k+1} = w_k x_k + (1 - w_k) \bar{x}_k.
            \end{array} \right. 
\end{eqnarray}
Then, we have \textit{almost surely} that $f(\bar{x}_k) - f_* = o\br{\frac{1}{\sum_{t=0}^{k-1} \eta_t^{\max}\gamma_t}}.$
\end{theorem}
We also present upper bounds on the suboptimality for \ref{eq:SGD-ALS} and \ref{eq:SGD-PS} in expectation, from which we can derive convergence rates.
\begin{theorem}\label{thm:adap_nonasymp}
Let Assumptions \ref{asm:smoothconvex} and \ref{asm:lower-bound} hold. Let $\br{\eta_k^{\max}}$ and $\br{\gamma_k}$ be strictly positive, decreasing sequences with $\gamma_k \leq c,$ for all $k \in \N$ and $c \geq \frac{1}{2}$. Then the iterates of \ref{eq:SGD-ALS} and \ref{eq:SGD-PS} satisfy
\begin{align}\label{eq:adap_non_asymp}
\ec{f(\bar{x}_k) - f_*} \leq \frac{2 c a_0\sqn{x_0 - x_*} + 4c\sum_{t=0}^{k-1}\gamma_t\eta_t^{\max}\br{\frac{\eta_t^{\max}\cL}{2(1-c)} - 1}_+\bar{\sigma}^2 + 2\sum_{t=0}^{k-1}\gamma_t^2\eta_t^{\max}\bar{\sigma}^2}{\sum_{t=0}^{k-1}\gamma_t\eta_t^{\max}},
\end{align}\\[-0.6cm]
where $\bar{x}_k = \sum_{t=0}^{k-1}\frac{\eta_t^{\max}\gamma_t}{\sum_{j=0}^{k-1}\eta_j^{\max} \gamma_j} x_t$ and $a_0 = \max\left\{\frac{\eta_0^{\max}\cL}{2(1-c)},\, 1\right\}$.
\end{theorem}

We now give precise convergence rates derived from the two previous theorems, in the overparametrized as well as the non-overparametrized cases.
\begin{corollary}[Corollary of Theorems \ref{thm:adap_asymp} and \ref{thm:adap_nonasymp}]\label{cor:stepsizes_rates_sgd_adap}
Let $\epsilon, \eta, \gamma > 0$, with $\gamma \leq c$. If  $\eta_k^{\max} = \eta k^{-\frac{4\epsilon}{3}}$ and $\gamma_k = \gamma k^{-\frac{1}{2} + \frac{\epsilon}{3}}$
\begin{align}\label{eq:adap_asymp}
f(\bar{x}_k) - f_* = o\br{\frac{1}{k^{\frac{1}{2} - \epsilon}}} \; a.s.
\quad \mbox{and} \quad\ec{f(\bar{x}_k) - f_*} = O\br{\frac{1}{k^{\frac{1}{2} - \epsilon}}}.
\end{align}
If $\bar{\sigma}^2 = 0$. Then, setting $\eta_k^{\max} = \eta > 0$, $c=\frac{2}{3}$ and $\gamma_k = 1$, then for all $x_* \in \X_*$,
\begin{align}
f(\bar{x}_k) - f_* = o\br{\frac{1}{k}} \; a.s. \quad \mbox{and} \quad \ec{f(\bar{x}_k) - f_*} \leq  \frac{2\max\left\{\frac{3\eta L_{\max}}{2}, 1\right\}\sqn{x_0 - x_*}}{\eta k}.
\end{align}
\end{corollary}
Notice from \eqref{eq:adap_non_asymp} and \eqref{eq:adap_asymp} that our analysis highlights a tradeoff between the asymptotic and the nonasymptotic convergence in expectation of \ref{eq:SGD-ALS} and \ref{eq:SGD-PS}. Indeed, \eqref{eq:adap_asymp} predicts that the slower the convergence of $\br{\eta_k^{\max}}_k$ towards $0$ (as $\epsilon \rightarrow 0$), the better is the resulting asymptotic convergence rate. However, according to \eqref{eq:adap_non_asymp}, if $\br{\eta_k^{\max}}_k$ vanishes slowly, the second term on the right hand side of \eqref{eq:adap_asymp} vanishes slowly as well, which makes the bound in \eqref{eq:adap_non_asymp} looser.

Notice also that to be able to derive convergence rates in the non-overparametrized case from the previous theorem, we not only decrease the maximum step sizes, but also scale the adaptive step size $\eta_k$ by multiplying it by a decreasing sequence $\gamma_k$.

\section{\textit{Almost sure }convergence rates for Stochastic Heavy Ball}\label{sec:shb_asymptotic}
The rates we derived for \ref{eq:SGD}, \ref{eq:SGD-ALS} and \ref{eq:SGD-PS} in the previous section all hold at some weighted average of the iterates. Yet, in practice, it is the last iterate of SGD which is used. In contrast, we show that these rates hold for \emph{the last iterate} of SHB, which is due to the online averaging inherent to SHB that we highlight in Proposition \ref{prop:ima}. We present the first \textit{almost sure} convergence rates for SHB, and also show that the deterministic HB converges at a $o(1/k)$ rate, which is asymptotically faster than the previously established $O(1/k)$ \citep{Ghadimi2014}.

We now present \textit{almost sure} convergence rates for SHB. The proof of this result is inspired by ideas from \cite{Chambolle15}, who prove the convergence of the iterates of FISTA \citep{beck2009fast} and \cite{Attouch16}, who prove the $o(1/k^2)$ convergence of FISTA.
\begin{theorem}\label{thm:as_conv_shb}
Let  $x_{-1} = x_0$ and consider the iterates of~\ref{eq:SHB_IMA}. Let Assumption~\ref{asm:smoothconvex} hold.  
Let $\eta_k$ be a sequence of stepsizes which verifies Condition~\ref{con:step_sizes} and $\forall k \in \N, 0 < \eta_k \leq 1/8\cL$. If
\begin{align}\label{eq:lambdakasymptotic}
\lambda_0 = 0 \quad \mbox{and}\quad  \lambda_k = \frac{\sum_{t=0}^{k-1} \eta_t}{4\eta_k} \quad  \mbox{for all }k \in \N^*,
\end{align}  
then we have \textit{almost surely} that $x_k \underset{k \rightarrow +\infty}{\rightarrow} x_*\;$  for some $x_* \in \X_*$, and ${f(x_k) - f_* = o\br{\frac{1}{\sum_{t=0}^{k-1} \eta_t}}}$.
\end{theorem}

Note that when specialized to full gradients sampling, \textit{i.e.} when we use the deterministic HB method, our results hold without the need for \textit{almost sure} statements.

To the best of our knowledge, Theorem \ref{thm:as_conv_shb} is the first result showing that the iterates of SHB converge to a minimizer assuming only smoothness and convexity. Note that this result is not directly comparable to \cite{Gadat18}, who study the more general nonconvex setting but  use assumptions beyond smoothness.

In the general stochastic setting, Theorem \ref{thm:as_conv_shb} shows that SHB enjoys the same \textit{almost sure} convergence rates as SGD with averaging (See Table \ref{tab:summary_rates}). However, an added benefit of SHB is that these rates hold for the last iterate, which conforms to what is done in practice.

\begin{corollary}\label{cor:faster1k}
Assume $\sigma^2 = 0$ and let $\eta_k = \eta < 1/4\cL$ for all $k \in \N$. By Theorem \ref{thm:as_conv_shb} we have $$\lim_k k \br{f(x_k) - f_*} = 0, \qquad \mbox{\emph{almost surely}}.$$
\end{corollary}
This corollary has fundamental implications in the deterministic and the stochastic case. In the stochastic case, it shows that when $\sigma^2 = 0$, \ref{eq:SHB_IMA} with a fixed step size converges at a $o(1/k)$ rate at the last iterate. In the deterministic case, $\sigma^2 =0$ always holds, as at each iteration we use the true gradient $\nabla f(x_k)$, and we have $\nabla f(x_*) = 0$ for all $x_* \in \cX_*$. Thus  Corollary~\ref{cor:faster1k} shows that the HB method enjoys the same $o(1/k)$ asymptotic convergence rate as gradient descent \citep{Lee19}.

It seems that it is our choice iteration-dependent momentum coefficients given by~\eqref{eq:alphafrometalambdaintro} and~\eqref{eq:lambdakasymptotic}  that enable this fast  `small o' convergence of the objective values for SHB.  Recent work by~\cite{Attouch16} corroborates with this finding, where the authors also showed that a version of (deterministic) Nesterov's Accelerated Gradient algorithm
with carefully chosen  iteration dependent momentum coefficients converges at a $o(1/k^2)$ rate, rather than the previously known $O(1/k^2)$.

\section{Non-convex \textit{almost sure} convergence rates for SGD}\label{sec:sgd_nonconvex}
We now move on to the non-convex case,  where we use the following assumption from \cite{Khaled20}.
\begin{assumption}
There exist constants $A, B, C \geq 0$ s.t. for all $x \in \R^d$,
\begin{align}\label{eq:ABC}
\ec[v]{\sqn{\nabla f_{v}(x)}} \leq A\br{f(x) - f_*} + B\sqn{\nabla f(x)} + C.  \tag{ABC}
\end{align}
\end{assumption}
This assumption is called \emph{Expected Smoothness} in \cite{Khaled20}. It includes the bounded gradients assumption, with $A=B=0$ and $C = G > 0$, and  the bounded gradient variance assumption, with $A=0$, $B=1$ and $C=\sigma^2$, as special cases. See \cite[Th. 1]{Khaled20} for a thorough investigation of the other assumptions used in the litterature which are implied by \eqref{eq:ABC}. A major benefit of this assumption is that when $f$ is a finite-sum (Remark \ref{rem:finite_sum}) and the $f_i$ functions are lower-bounded, \eqref{eq:ABC} always holds~\citep[Prop. 3]{Khaled20}.
\begin{remark}[\cite{Khaled20}, Prop. 3]
In the setting of Remark \ref{rem:finite_sum}, and assuming that for all $i \in [n], \; f_i \geq f_i^* > -\infty$, Assumption \eqref{eq:ABC} holds with:
\begin{equation}
A = \frac{n-b}{b(n-b)}L_{\max},  \quad B = \frac{n(b-1)}{b(n-1)}, \quad \mbox{and} \quad C = \frac{2A}{n}\sum_{i=1}^n \br{f_* - f_i^*} .
\end{equation}
\end{remark}

Since a global minimizer of $f$ does not always exist in the nonconvex case, we can now only hope to find a stationary point. Hence, we present asymptotic convergence rates for the squared gradient norm.

\begin{theorem}\label{thm:as_conv_sgd_nonconvex}
Consider the iterates of \ref{eq:SGD}. Assume that \eqref{eq:ABC} holds. Choose stepsizes which verify Condition \ref{con:step_sizes} (with $C$ in place of $\sigma^2$) such that $\forall k \in \N, \; 0 < \eta_k \leq 1/(BL)$.  Then, we have $a.s$ that 
\[\underset{t=0,\dots,k-1}{\min} \sqn{\nabla f(x_t)} = o\br{\frac{1}{\sum_{t=0}^{k-1} \eta_t}}.\]
\end{theorem}
From this result, we can derive \textit{almost sure} convergence rates arbitrarily close to $o(1/\sqrt{k})$, which can be improved to $o(1/k)$ in the overparametrized setting (See Table \ref{tab:summary_rates}). Since these results are similar to Corollary \ref{cor:stepsizes_rates_sgd}, we omit them for brievity and report them in Table \ref{tab:summary_rates} and Corollary \ref{cor:stepsizes_rates_sgd_nonconvex} in Appendix \ref{sec:app_formal_statements}.

\section{Experiments}\label{sec:experiments} 
In our experiments, we aimed to examine whether or not \ref{eq:SHB_IMA} with the parameter settings suggested by our theory performed better than \ref{eq:SGD} and SGD with three common alternative parameter settings used throughout the machine learning literature: SGD with fixed momentum $\beta$ of 0.9 and 0.99 as well as no momentum.

For our experiments, we selected a diverse set of multi-class classification problems from the LibSVM repository, 25 problems in total. These datasets range from a few classes to a thousand, and they vary from hundreds of data-points to hundreds of thousands. We normalized each dataset by a constant so that the largest data vector had norm $1$. We used a multi-class logistic regression loss with no regularization so we could test the non-strongly convex convergence properties, and we ran for 50 epochs with no batching.

We use SHB to denote the method~\eqref{eq:SHB} with $\alpha_k$ and $\beta_k$ set using~\eqref{eq:alphafrometalambdaintro} (or equivalently the method~\eqref{eq:SHB_IMA}) and we left $\eta$, as well as the step sizes of all the methods we compare, as a constant to be determined through grid search. For the gridsearch, we used power-of-2 grid ($2^i$), we ran 5 random seeds and chose the learning rate that gave the lowest loss on average for each combination of problem and method. We widened the grid search as necessary for each combination to ensure that the chosen learning-rate was not from the endpoints of our grid search. Although it is possible to give a closed-form bound for the Lipschitz smoothness constant for our test problems, the above setting is less conservative and has the advantage of being usable without requiring any knowledge about the problem structure. 

 We then ran 40 different random seeds to produce Figure~\ref{fig:experiments}. To determine which method, if any, was best on each problem, we performed t-tests with Bonferroni correction, and we report how often each method was statistically significantly superior to all of the other three methods in Table \ref{tbl:best}. The stochastic heavy ball method using our theoretically motivated parameter settings performed better than all other methods on 11 of the 25 problems. On the remaining problems, no other method was statistically significantly better than all of the rest.
 \afterpage{
  \begin{table*}[h!]
\centering \begin{tabular}{|c|c|c|c|c|c|}
\hline 
 & SHB & SGD & Momentum 0.9 & Momentum 0.99 & No best method\tabularnewline
\hline 
\hline 
Best method for & 11 & 0 & 0 & 0 & 14\tabularnewline
\hline 
\end{tabular}
\caption{\label{tbl:best}Count of how many problems each method is statistically significantly superior to the rest on}
\end{table*}
\begin{figure}[t!]
\includegraphics[width=\textwidth, height=0.5\textheight]{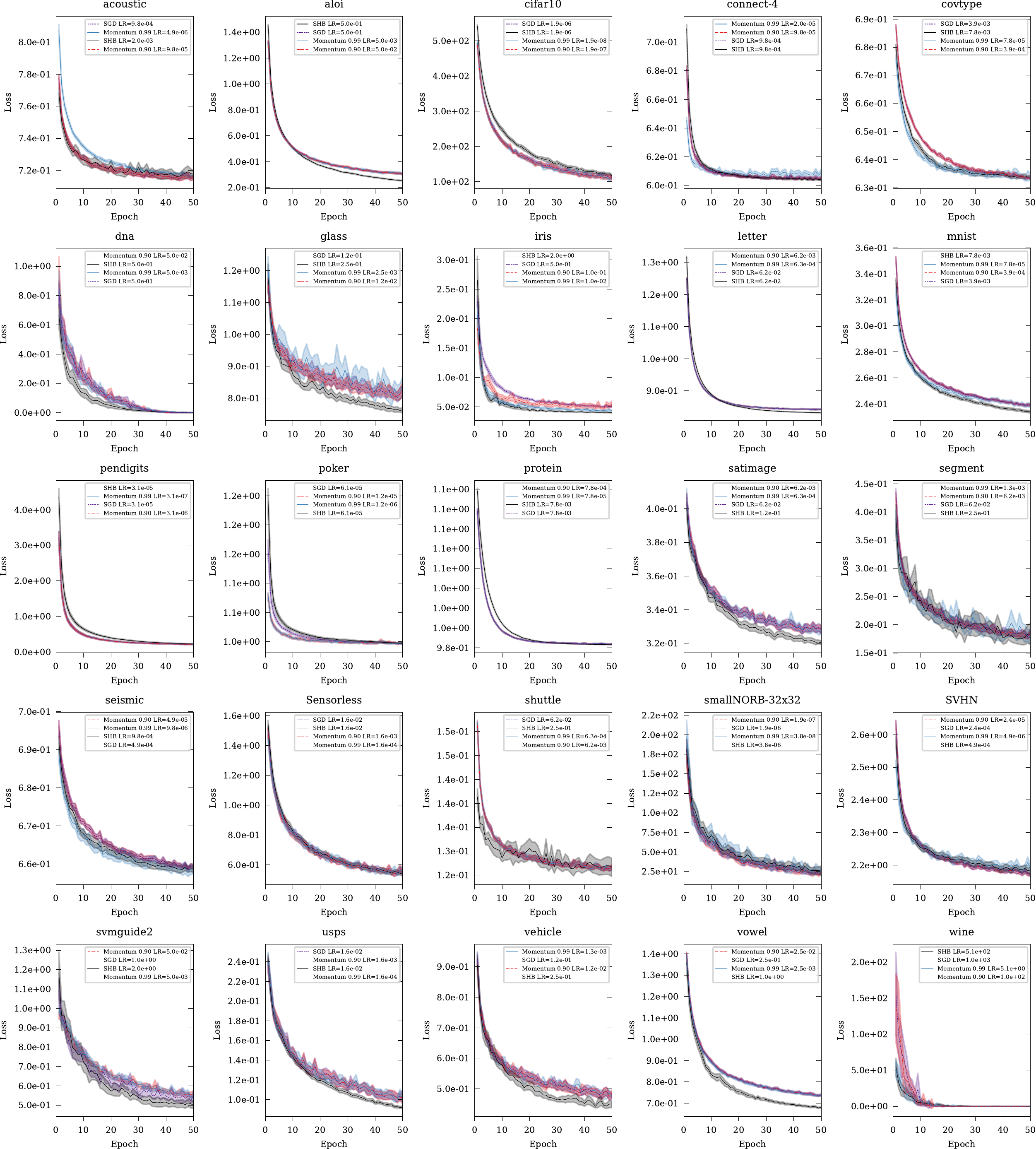}
\caption{\label{fig:experiments}Average training error convergence plots for 25 LibSVM datasets, with using the best learning rate for each method and problem combination. Averages are over 40 runs. Error bars show a range of +/- 2SE.}
\end{figure}
} 

\subsection*{Acknowledgements}
The work of Othmane Sebbouh was supported in part by the French government under management of Agence Nationale de la Recherche as part of the "Investissements d’avenir" program, reference ANR19-P3IA-0001 (PRAIRIE 3IA Institute). Othmane Sebbouh also acknowledges the support of a "Chaire d'excellence de l'IDEX Paris Saclay".

\newpage
\bibliographystyle{plainnat}
\bibliography{biblio.bib}

\newpage
\onecolumn 

\begin{appendices}

\renewcommand{\thesection}{\Alph{section}}
\begin{center}
  \normalsize\bfseries\MakeUppercase{APPENDIX}
\end{center}

\tableofcontents

\pagebreak
The appendix is organized as follows:
\begin{itemize}
\item In Section \ref{sec:app_formal_statements} we present the proofs of Remark \ref{rem:stepsizes_condition} and the stepsize choices and the corresponding convergence rates derived in the corollaries listed in Table \ref{tab:summary_rates}.
\item In Section \ref{sec:app_proofs_intro} we present proofs for Section \ref{sec:introduction}.
\item In Section \ref{sec:app_proofs_as_conv_sgd} we present proofs for Section \ref{sec:sgd_asymp}.
\item In Section \ref{sec:app_proofs_shb} we present proofs for Section \ref{sec:shb_asymptotic}.
\item In Section \ref{sec:app_proofs_as_conv_sgd_nonconvex} we present proofs for Section \ref{sec:sgd_nonconvex}.
\item In Section \ref{sec:app_nonsmooth}, we present our results for the convergence of stochastic subgradient descent under the bounded gradients assumptions.
\item In Section \ref{sec:app_shb_exp}, we present convergence rates for SHB in expectation without the bounded gradients and bounded gradient variance assumptions.
\end{itemize}

\section{Proofs of corollaries on convergence rates and stepsize choices}\label{sec:app_formal_statements}

\subsection{Proof of remark \ref{rem:stepsizes_condition}}
\begin{proof}
Let $\eta_k = \frac{\eta}{k^{\xi}}$ with $\eta > 0$ and $\xi \in [0, 1)$.  Clearly, $(\eta_k)_k$ is decreasing and $\sum_k \eta_k = \infty$. And we have
\begin{align}
\sum_{t=0}^{k-1}\eta_t \sim \eta k^{1 - \xi}.
\end{align}
Hence,
\begin{align}
\frac{\eta_k}{\sum_{t=0}^{k-1}\eta_t} \sim \frac{1}{k}.
\end{align}
Hence, $\sum_k \frac{\eta_k}{\sum_{t=0}^{k-1}\eta_t} = \infty$.
\begin{itemize}
\item If $\sigma^2 \neq 0$. Let $\xi \in (\frac{1}{2}, 1]$.  Then $\sum_k \eta_k^2 < \infty$, and the stepsizes verify Condition \ref{con:step_sizes}.
\item If $\sigma^2 = 0$. Let $\xi \in [0, 1)$.  We have $\sum_k \eta_k^2\sigma^2 = 0 < \infty$. Hence, the stepsizes verify Condition~\ref{con:step_sizes}.
\end{itemize}
\end{proof}

\subsection{SGD: Proof of Corollary \ref{cor:stepsizes_rates_sgd} }
\begin{proof}
\begin{itemize}
\item If $\sigma^2 \neq 0$.  Let $\eta_k = \frac{\eta}{k^{1/2 + \epsilon}}$. From Remark \ref{rem:stepsizes_condition}, we have that the stepsizes verify Condition \ref{con:step_sizes}. Moreover, $\sum_k \frac{1}{k^{\frac{1}{2} + \epsilon}} \sim k^{-1/2 + \epsilon}$. Thus, from Theorem \ref{thm:as_conv_sgd}:
\begin{align}
f(\bar{x}_k) - f_* = o\br{\frac{1}{k^{1/2 - \epsilon}}}.
\end{align}
\item If $\sigma^2 = 0$.  Let $\eta_k = \eta$. From Remark \ref{rem:stepsizes_condition}, the stepsizes verify Condition \ref{con:step_sizes}. Thus, from Theorem~\ref{thm:as_conv_sgd}:
\begin{align}
f(\bar{x}_k) - f_* = o\br{\frac{1}{k}}.
\end{align}
\end{itemize}

\end{proof}

\subsection{SGD with adaptive step sizes: Proof of Corollary \ref{cor:stepsizes_rates_sgd_adap} }
\begin{proof}
Let $\eta, \gamma, c > 0$. We first prove the \emph{almost sure} convergence results.
\begin{itemize}
\item If $\bar{\sigma}^2 \neq 0$.  Let $\eta_k^{\max} = \eta k^{-4\epsilon/3}$ and $\gamma_k = \gamma k^{-\frac{1}{2} + \epsilon/3}$. 
Clearly, $(\eta_k^{\max} \gamma_k)_k$ is decreasing, $\sum_k \eta_k = \infty$ and $\sum_k \eta_k^{\max} \gamma_k^2 < \infty$. And
\begin{align}
\sum_{t=0}^{k-1}\eta_t^{\max}\gamma_t \sim \eta\gamma k^{1/2 - \epsilon}.
\end{align}
Hence,
\begin{align}
\frac{\eta_k^{\max} \gamma_k }{\sum_{t=0}^{k-1}\eta_t^{\max}\gamma_t} \sim \frac{1}{k}.
\end{align}
Hence, $\sum_k\frac{\eta_k^{\max} \gamma_k }{\sum_{t=0}^{k-1}\eta_t^{\max}\gamma_t}  = \infty$. Thus, the stepsizes verify the conditions of Theorem \ref{thm:adap_asymp}, and we have
\begin{align}
f(\bar{x}_k) - f_* = o\br{\frac{1}{k^{1/2 - \epsilon}}}.
\end{align}
\item If $\bar{\sigma}^2 = 0$.  Let $\eta_k^{\max} = \eta$ and $\gamma_k = 1$. Clearly $(\eta_k^{\max}\gamma_k)_k$ is decreasing since it is constant, $\sum_k \eta_k^{\max}\gamma_k = \infty$, $\sum_k \eta_k^{\max}\gamma_k^2\bar{\sigma}^2 = 0 < \infty$, and $\sum_k \frac{\eta_k^{\max}\gamma_k}{\sum_{t=0}^{k-1}\eta_t^{\max}\gamma_t} = \sum_k 1 =  \infty$. Thus, the stepsizes verify the conditions of Theorem \ref{thm:adap_asymp}, and we have
\begin{align}
f(\bar{x}_k) - f_* = o\br{\frac{1}{k}}.
\end{align}
\end{itemize}

We now prove the convergence rates in expectation. Remember that from Theorem \ref{thm:adap_nonasymp}, we have that if $\br{\eta_k^{\max}}$ and $\br{\gamma_k}$ are strictly positive, decreasing sequences with $\gamma_k \leq c$ for all $k \in \N$ and $c \geq \frac{1}{2}$, then the iterates of \ref{eq:SGD-ALS} and \ref{eq:SGD-PS} satisfy
\begin{align}\label{eq:adap_non_asymp_app}
\ec{f(\bar{x}_k) - f_*} \leq \frac{2 c a_0\sqn{x_0 - x_*} + 4c\sum_{t=0}^{k-1}\gamma_t\eta_t^{\max}\br{\frac{\eta_t^{\max}\cL}{2(1-c)} - 1}_+\bar{\sigma}^2 + 2\sum_{t=0}^{k-1}\gamma_t^2\eta_t^{\max}\bar{\sigma}^2}{\sum_{t=0}^{k-1}\eta_t^{\max}\gamma_t},
\end{align}
where $\bar{x}_k = \sum_{t=0}^{k-1}\frac{\eta_t^{\max}\gamma_t}{\sum_{j=0}^{k-1}\eta_j^{\max} \gamma_j} x_t$ and $a_0 = \max\left\{\frac{\eta_0^{\max}\cL}{2(1-c)},\, 1\right\}$.
\begin{itemize}
\item If $\bar{\sigma}^2\neq 0$.  Let $\eta_k^{\max} = \eta k^{-4\epsilon/3}$ and $\gamma_k = \gamma k^{-\frac{1}{2} + \epsilon/3}$
Since $\eta_t^{\max} \rightarrow 0$, there exists $k_0 \in \N$ such that for all $t \geq k_0, \; \gamma_t\eta_t^{\max}\br{\frac{\eta_t^{\max}\cL}{2(1-c)} - 1}_+ = 0$. Hence, for all $k \geq k_0$
\begin{align}
\ec{f(\bar{x}_k) - f_*} \leq \frac{2 c a_0\sqn{x_0 - x_*} +4c\sum_{t=0}^{k_0-1}\gamma_t\eta_t^{\max}\br{\frac{\eta_t^{\max}\cL}{2(1-c)} - 1}_+\bar{\sigma}^2  +  2\sum_{t=0}^{k-1}\gamma_t^2\eta_t^{\max}\bar{\sigma}^2}{\sum_{t=0}^{k-1}\eta_t^{\max}\gamma_t}.
\end{align}
Replacing $\eta_k^{\max }$ and $\gamma_k$ with their values gives
\begin{align}
\ec{f(\bar{x}_k) - f_*}  = O\br{\frac{1}{k^{1/2- \epsilon}}}.
\end{align}
\item If $\bar{\sigma}^2 = 0$. Let $\eta_k^{\max} = \eta > 0$ and $\gamma = 1$. We have from \eqref{eq:adap_non_asymp_app} that for all $k \in \N$,
\begin{align}
\ec{f(\bar{x}_k) - f_*} \leq  \frac{2\max\left\{\frac{3\eta \cL}{2}, 1\right\}\sqn{x_0 - x_*}}{\eta k}.
\end{align}
\end{itemize}
\end{proof}

\subsection{SHB}
\begin{corollary}[Corollary of Theorem \ref{thm:as_conv_shb}]\label{cor:stepsizes_rates_shb}
Let Assumption \ref{asm:smoothconvex} hold. Let $0 < \eta \leq 1/4\cL$ and $\epsilon > 0$. 
\begin{itemize}
\item if $\sigma^2 \neq 0$.  Let $\eta_k = \frac{\eta}{k^{1/2 + \epsilon}}$. Then
\begin{align}
f(x_k) - f_* = o\br{\frac{1}{k^{1/2 - \epsilon}}}.
\end{align}
\item If $\sigma^2 =  0$.  Let $\eta_k = \eta $. Then
\begin{align}
f(x_k) - f_* = o\br{\frac{1}{k}}.
\end{align} 
\end{itemize}
\end{corollary}

\begin{proof}
The proof is the same as the proof of Corollary \ref{cor:stepsizes_rates_sgd}, using Theorem~\ref{thm:as_conv_shb} instead of Theorem~\ref{thm:as_conv_sgd}.
\end{proof}

\subsection{SGD, nonconvex}
\begin{corollary}[Corollary of Theorem \ref{thm:as_conv_sgd_nonconvex}]\label{cor:stepsizes_rates_sgd_nonconvex}
Let Assumption \eqref{eq:ABC} hold. Let $0 < \eta \leq 1/4\cL$ and $\epsilon > 0$. 
\begin{itemize}
\item If $\sigma^2 = 0$.  Let $\eta_k = \frac{\eta}{k^{1/2 + \epsilon}}$. Then,
\begin{align}
\min_{t=0,\dots,k-1} \sqn{\nabla f(x_k)} = o\br{\frac{1}{k^{1/2 - \epsilon}}}.
\end{align}
\item If $\sigma^2 \neq  0$.  Let $\eta_k = \eta $. Then,
\begin{align}
\min_{t=0,\dots,k-1} \sqn{\nabla f(x_k)} = o\br{\frac{1}{k}}.
\end{align}
\end{itemize}
\end{corollary}

\begin{proof}
\begin{itemize}
\item If $C \neq 0$.  Let $\eta_k = \frac{\eta}{k^{1/2 + \epsilon}}$. From Remark \ref{rem:stepsizes_condition}, we have that the stepsizes verify Condition \ref{con:step_sizes} with $C$ in place of $\sigma^2$. Moreover, $\sum_k \frac{1}{k^{\frac{1}{2} + \epsilon}} \sim k^{-1/2 + \epsilon}$. Thus, from Theorem \ref{thm:as_conv_sgd_nonconvex}:
\begin{align}
\min_{t=0,\dots,k-1} \sqn{\nabla f(x_k)}  = o\br{\frac{1}{k^{1/2 - \epsilon}}}.
\end{align}
\item If $C = 0$.  Let $\eta_k = \eta$. From Remark \ref{rem:stepsizes_condition}, the stepsizes verify Condition \ref{con:step_sizes} with $C$ in place of $\sigma^2$. Thus, from Theorem~\ref{thm:as_conv_sgd_nonconvex}:
\begin{align}
\min_{t=0,\dots,k-1} \sqn{\nabla f(x_k)}  = o\br{\frac{1}{k}}.
\end{align}
\end{itemize}
\end{proof}



\section{Proofs for Section \ref{sec:introduction}} \label{sec:app_proofs_intro}

\subsection{Proof of Lemma \ref{lem:smoothconvex}}
\begin{proof}
Since for all $v \sim \D$, $f_v$ is convex and $L_v$ smooth, we have from \cite[Equation 2.1.7]{Nesterov-convex} that
\begin{align}
\sqn{\nabla f_v(x) - \nabla f_v(x_*)} &\leq 2L_v\br{f_v(x) - f_v(x_*) - \langle \nabla f_v(x_*), x - x_* \rangle}\\
& \overset{\mbox{Asm. \ref{asm:smoothconvex}}}{\leq} 2\cL\br{f_v(x) - f_v(x_*) - \langle \nabla f_v(x_*), x - x_* \rangle}
\end{align}
Hence,
\begin{align}
\ec[v\sim\D]{\sqn{\nabla f_v(x) - \nabla f_v(x_*)}} \leq 2\cL \br{f(x) - f(x_*)}.
\end{align}
Therefore,
\begin{align}
\ec[v\sim\D]{\sqn{\nabla f_v(x)}} & \leq  2\ec[v\sim\D]{\sqn{\nabla f_v(x) - \nabla f_v(x_*)}} + 2 \ec[v\sim\D]{\sqn{\nabla f_v(x_*)}} \\
& \leq 4\cL \br{f(x) - f(x_*)} + 2 \ec[v\sim\D]{\sqn{\nabla f_v(x_*)}}.
\end{align}
\end{proof}

\subsection{Proof of Proposition~\ref{prop:ima}}\label{sec:proof_ima}

%

\begin{proof}
Consider the iterate-averaging method
\begin{eqnarray}
z_{k+1} & = & z_{k}-\eta_{k}\nabla f_{v_k}(x_{k}), \label{eq:zup}\\
x_{k+1} & = & \frac{\lambda_{k+1}}{\lambda_{k+1}+1}x_{k}+\frac{1}{\lambda_{k+1}+1}z_{k+1},\label{eq:xupz}
\end{eqnarray}
and let
\begin{align}
\alpha_k \;= \; \frac{\eta_{k}}{\lambda_{k+1}+1}  \quad \mbox{and} \quad \beta_k \;=\;  \frac{\lambda_{k}}{\lambda_{k+1}+1}. \label{eq:alphafrometalambda}
\end{align}
Substituting~\eqref{eq:zup} into \eqref{eq:xupz} gives
\begin{equation}\label{eq:s84j8sj4}
 x_{k+1}  =  \frac{\lambda_{k+1}}{\lambda_{k+1}+1}x_{k}+\frac{1}{\lambda_{k+1}+1}\left( z_{k}-\eta_{k}\nabla f_{v_k}(x_{k})\right). 
\end{equation}
Now using~\eqref{eq:xupz} at the previous iteration we have that that
 \[z_{k} =(\lambda_{k}+1) \left(x_k -  \frac{\lambda_{k}}{\lambda_{k}+1}x_{k-1}\right) =(\lambda_{k}+1) x_k - \lambda_k x_{k-1}  .\]
 Substituting the above into~\eqref{eq:s84j8sj4} gives
 \begin{align}
  x_{k+1} & =  \frac{\lambda_{k+1}}{\lambda_{k+1}+1}x_{k}+\frac{1}{\lambda_{k+1}+1}\left( (\lambda_{k}+1) x_k - \lambda_k x_{k-1}-\eta_{k}\nabla f_{v_k}(x_{k})\right)\\
  &= x_k  - \frac{\eta_{k}}{\lambda_{k+1}+1}\nabla f_{v_k}(x_{k}) +  \frac{\lambda_{k}}{\lambda_{k+1}+1}\left(x_{k}-x_{k-1}\right).
 \end{align}
Consequently, \eqref{eq:alphafrometalambda} gives the desired expression.
\end{proof}

\section{Proofs for Section \ref{sec:sgd_asymp}} \label{sec:app_proofs_as_conv_sgd}
\subsection{Proof of Theorem \ref{thm:as_conv_sgd}}

\begin{proof}
Consider the setting of Theorem \ref{thm:as_conv_sgd}. Expanding the squares we have that
\begin{align}
\sqn{x_{k+1} - x_*} = \sqn{x_{k} - x_*} - 2 \eta_k \langle \nabla f_{v_k}(x_k), x_k - x_* \rangle + \eta_k^2\sqn{\nabla f_{v_k}(x_k)}.
\end{align}
Then taking conditional expectation $\ec[k]{\cdot} \eqdef \ec{\cdot \; | \; x_k}$ gives
\begin{eqnarray*}
\ec[k]{\sqn{x_{k+1} - x_*}} &=& \sqn{x_{k} - x_*} - 2 \eta_k \langle \nabla f(x_k), x_k - x_* \rangle + \eta_k^2\ec[k]{\sqn{\nabla f_{v_k}(x_k)}} \\
&\overset{\eqref{eq:convexity} + \eqref{eq:expsmooth}}{\leq} & \sqn{x_{k} - x_*} - 2 \eta_k \br{1 - 2\eta_k \cL}\br{f(x_k) - f_*} + 2\eta_k^2\sigma^2.
\end{eqnarray*}
Since $\eta_k \leq \frac{1}{4\cL}$, we have that $1 - 2\eta_k \cL \geq \frac{1}{2}$. Hence, rearranging, we have
\begin{align}\label{eq:main_rec_sgd}
\ec[k]{\sqn{x_{k+1} - x_*}} + \eta_k \br{f(x_k) - f_*} \leq \sqn{x_{k} - x_*} + 2\eta_k^2\sigma^2.
\end{align}
Using Lemma \ref{lem:simple_RS}, we have that $\br{\sqn{x_k - x_*}}_k$ converges \textit{almost surely}.

From~\eqref{eq:w_xbar} we have that $w_k = \frac{2\eta_k}{\sum_{j=0}^{k}\eta_j}$. Since $w_0 = 2\frac{\eta_0}{\eta_0}=2$ we have that $\bar{x}_1 = 2 x_0 - \bar{x}_0 = 2x_0 - x_0 = x_0$. Hence, it holds that
\begin{align}\label{eq:init_verification_weights}
f(\bar{x}_1) - f_* = f(x_0) - f_* = w_0 \br{f(x_0) - f_*} + (1 - w_0)\br{f(\bar{x}_0) - f_*} .
\end{align}

Now for $k \in \N^*$ we have that following equivalence 
\[w_k \in [0,1] \quad  \iff \quad 2\eta_k \leq \sum_{j=0}^{k} \eta_j  \quad  \iff \quad \eta_k \leq \sum_{j=0}^{k-1}\eta_j.\]

The right hand side of the equivalence holds because $\br{\eta_k}_k$ is a decreasing sequence.
 Hence, by Jensen's inequality, we have $\forall k \in \N^*,$
\begin{align}
f(\bar{x}_{k+1}) - f_* \leq w_k \br{f(x_k) - f_*} + (1-w_k) \br{f(\bar{x}_k) - f_*}.
\end{align}

Together with \eqref{eq:init_verification_weights}, this shows that the last inequality holds for all $k \in \N$. Thus,
\begin{align}
\eta_k\br{f(x_k) - f_*} \geq \frac{\eta_k}{w_k}\br{f(\bar{x}_{k+1}) - f_*} - \eta_k\br{\frac{1}{w_k} - 1}\br{f(\bar{x}_{k}) - f_*}.
\end{align}
Replacing this expression in \eqref{eq:main_rec_sgd} gives:
\begin{align}
&\ec[k]{\sqn{x_{k+1} - x_*}} +  \frac{\eta_k}{w_k}\br{f(\bar{x}_{k+1}) - f_*}\\
 &\quad \quad \leq \sqn{x_{k} - x_*} + \eta_k\br{\frac{1}{w_k} - 1}\br{f(\bar{x}_{k}) - f_*}
+  2\eta_k^2\sigma^2.
\end{align}
 Hence substituting in the definition of $w_k$ from~\eqref{eq:w_xbar} gives
\begin{align}
&\ec[k]{\sqn{x_{k+1} - x_*}} +  \frac{1}{2}\sum_{j=0}^{k}\eta_j\br{f(\bar{x}_{k+1}) - f_*}\\
 &\quad\quad\leq \sqn{x_{k} - x_*} + \frac{1}{2}\br{\sum_{j=0}^{k-1}\eta_j - \eta_k}\br{f(\bar{x}_{k}) - f_*}
+  2\eta_k^2\sigma^2.
\end{align}
Thus re-arranging
\begin{align}
&\ec[k]{\sqn{x_{k+1} - x_*}} +  \frac{1}{2}\sum_{j=0}^{k}\eta_j\br{f(\bar{x}_{k+1}) - f_*} + \frac{\eta_k}{2}\br{f(\bar{x}_{k}) - f_*} \\
 &\quad\quad\leq \sqn{x_{k} - x_*} + \frac{1}{2}\sum_{j=0}^{k-1}\eta_j\br{f(\bar{x}_{k}) - f_*}
+  2\eta_k^2\sigma^2,
\end{align}
which, by Lemma \ref{lem:simple_RS}, has the three following consequences:
\begin{align}
&\br{\sqn{x_k - x_*}+ \sum_{j=0}^{k}\eta_j\br{f(\bar{x}_{k+1}) - f_*}}_k \; \mbox{converges \textit{almost surely}},\\
& \mbox{and} \; \sum_k \eta_k\br{f(\bar{x}_k) - f_*} < \infty.
\end{align}
And since $\br{\sqn{x_k - x_*}}_k$ converges \textit{almost surely}, we have that $\br{\sum_{j=0}^{k-1}\eta_j\br{f(\bar{x}_{k}) - f_*}}_k$ converges \textit{almost surely}.
Hence, we have that $\lim_k \frac{\eta_k}{\sum_{j=0}^{k-1} \eta_j} \sum_{j=0}^{k-1} \eta_j\br{f(\bar{x}_k) - f_*} = \lim_{k} \eta_k \br{f(\bar{x}_k) - f_*} = 0$. But since we assumed that $\sum_k \frac{\eta_k}{\sum_{j=0}^{k-1} \eta_j}$ diverges, this implies that $\lim_k \sum_{j=0}^{k-1}\eta_j\br{f(\bar{x}_{k+1}) - f_*} = 0$, that is
\begin{align}
\boxed{f(\bar{x}_k) - f_* = o\br{\frac{1}{\sum_{j=0}^{k-1}\eta_j}}}
\end{align}
\end{proof}

\subsection{Proofs of Theorems \ref{thm:adap_asymp} and \ref{thm:adap_nonasymp}}
The results of Theorems \ref{thm:adap_asymp} and \ref{thm:adap_nonasymp} can be derived as corollaries of the following theorem.
\begin{theorem}\label{thm:rec_adap_step}
Let $\br{\eta_k^{\max}}$ and $\br{\gamma_k}$ be strictly positive, decreasing sequences with $\gamma_k \leq \frac{3c}{2}$ for all $k \in \N$ and $c \geq \frac{2}{3}$. Then the iterates of \eqref{eq:SGD-ALS} and the iterates of \eqref{eq:SGD-PS} satisfy
\begin{align}\label{eq:rec_adap_step}
2a_{k+1}\ec[k]{\sqn{x_{k+1} - x_*}}  + \eta_k^{\max}\gamma_k \br{f(x_k) - f_*}  &\leq 2a_k\sqn{x_k - x_*}\\
&+ 4\gamma_k\eta_k^{\max}\br{\frac{\eta_k^{\max}\cL}{2(1-c)} - 1}_+\bar{\sigma}^2 + \frac{2\gamma_k^2\eta_k^{\max}\bar{\sigma}^2}{c},
\end{align}
where $\bar{\sigma}^2 \eqdef f_* - \ec[v]{f_{v}^*}$, $a_k = \max\left\{\frac{\eta_k^{\max} \cL}{2(1-c)},\, 1\right\}$ and where $(b)_+ \eqdef \max(b, 0)$ for all $b \in \R$.
\end{theorem}

Before proving Theorem \ref{thm:rec_adap_step}, we need the following lemma:
\begin{lemma}[\cite{Vaswani19}, \cite{loizou2020stochastic}]\label{lem:lower_bound}
Let $g$ be an $L_g$-smooth function, $\alpha_{\max} > 0$ and $c \in [0, 1]$. If $\alpha \sim \mathbf{SALS}_{c, \alpha_{\max}}(g,x)$ or $\mathbf{SP}_{c, \alpha_{\max}}(g,x)$, then
\begin{align}
\min\left\{\frac{2(1-c)}{L_g},\, \alpha_{\max}\right\} \leq  \alpha \leq \alpha_{\max} \quad \mbox{and} \quad \alpha \sqn{\nabla g(x)} \leq \frac{g(x) - g^*}{c\sqn{\nabla g(x)}}. \label{eq:bnd_alpha_nabla_adap}
\end{align}
\end{lemma}

\begin{proof}[Proof of Theorem \ref{thm:rec_adap_step}]
Let us now prove Theorem \ref{thm:rec_adap_step}.
\begin{align}
\sqn{x_{k+1} - x_*} &= \sqn{x_k - x_*} - 2\eta_k\gamma_k \langle \nabla f_{v_k}(x_k), x_k - x_* \rangle + \eta_k^2\gamma_k^2\sqn{\nabla  f_{v_k}(x_k)}\\
&\overset{\eqref{eq:convexity}}{\leq} \sqn{x_k - x_*} - 2\eta_k\gamma_k \br{f_{v_k}(x_k) - f_{v_k}^*} + \eta_k^2\gamma_k^2\sqn{\nabla  f_{v_k}(x_k)}\\
&\overset{\eqref{eq:bnd_alpha_nabla_adap}}{\leq} \sqn{x_k - x_*} - 2\eta_k\gamma_k \br{f_{v_k}(x_k) - f_{v_k}(x_*)} + \frac{\eta_k\gamma_k^2}{c} \br{f_{v_k}(x_k) - f_{v_k}^*}\\
&= \sqn{x_k - x_*} - 2\eta_k\gamma_k\br{1 - \frac{\gamma_k}{2c}} \br{f_{v_k}(x_k) - f_{v_k}^*} +  2\eta_k\gamma_k\br{f_{v_k}(x_*) - f_{v_k}^*}.
\end{align}
Rearranging, we have
\begin{align}
2\eta_k\gamma_k\br{1 - \frac{\gamma_k}{2c}} \br{f_{v_k}(x_k) - f_{v_k}^*} &\leq \sqn{x_k - x_*} - \sqn{x_{k+1} - x_*} + 2\eta_k\gamma_k\br{f_{v_k}(x_*) - f_{v_k}^*}.
\end{align}
Define $\eta_k^{\min} \eqdef \min\left\{\frac{2(1-c)}{\cL},\, \eta_k^{\max}\right\}$. Then,
\begin{align}
2\eta_k^{\min}\gamma_k\br{1 - \frac{\gamma_k}{2c}} \br{f_{v_k}(x_k) - f_{v_k}^*} &\overset{\eqref{eq:bnd_alpha_nabla_adap}}{\leq} \sqn{x_k - x_*} - \sqn{x_{k+1} - x_*} + 2\eta_k\gamma_k\br{f_{v_k}(x_*) - f_{v_k}^*}.
\end{align}
Hence,
\begin{align} \label{eq:inter_before_multi_dec_sals_sp}
2\eta_k^{\min}\gamma_k\br{1 - \frac{\gamma_k}{2c}} \br{f_{v_k}(x_k) - f_{v_k}(x_*)} &\leq \sqn{x_k - x_*} - \sqn{x_{k+1} - x_*} \\
&+ 2\gamma_k\br{\eta_k -  \br{1 - \frac{\gamma_k}{2c}}\eta_k^{\min}}\br{f_{v_k}(x_*) - f_{v_k}^*}\\
&\leq \sqn{x_k - x_*} - \sqn{x_{k+1} - x_*} \\
&+ 2\gamma_k\br{\eta_k^{\max} -  \br{1 - \frac{\gamma_k}{2c}}\eta_k^{\min}}\br{f_{v_k}(x_*) - f_{v_k}^*}.
\end{align}
Notice that
\begin{align}
\frac{\eta_k^{\max}}{\eta_k^{\min}} = \max\left\{\frac{\eta_k^{\max}\cL}{2(1-c)}, 1\right\}.
\end{align}
Since $\br{\eta_k^{\max}}_k$ is decreasing, $\br{\frac{\eta_k^{\max}}{\eta_k^{\min}}}_k$ is decreasing as well. Hence, multiplying both sides of \eqref{eq:inter_before_multi_dec_sals_sp} by $\frac{\eta_k^{\max}}{\eta_k^{\min}}$,
\begin{align}
2\eta_k^{\max}\gamma_k\br{1 - \frac{\gamma_k}{2c}} \br{f_{v_k}(x_k) - f_{v_k}(x_*)} &\leq \frac{\eta_k^{\max}}{\eta_k^{\min}}\sqn{x_k - x_*} - \frac{\eta_{k+1}^{\max}}{\eta_{k+1}^{\min}}\sqn{x_{k+1} - x_*} \\
&+ 2\gamma_k\eta_k^{\max}\br{\frac{\eta_k^{\max}}{\eta_k^{\min}} - 1 + \frac{\gamma_k}{2c}}\br{f_{v_k}(x_*) - f_{v_k}^*}.
\end{align}
Hence, taking the expectation,
\begin{align}
2\eta_k^{\max}\gamma_k\br{1 - \frac{\gamma_k}{2c}} \br{f(x_k) - f_*} &\leq \frac{\eta_k^{\max}}{\eta_k^{\min}}\sqn{x_k - x_*} - \frac{\eta_{k+1}^{\max}}{\eta_{k+1}^{\min}}\ec[k]{\sqn{x_{k+1} - x_*}} \\
&+ 2\gamma_k\eta_k^{\max}\br{\frac{\eta_k^{\max}}{\eta_k^{\min}} - 1 + \frac{\gamma_k}{2c}}\bar{\sigma}^2.
\end{align}
where $\bar{\sigma}^2 \eqdef f_* - \ec[v]{f_{v}^*}$.
Using the fact that $1 - \frac{\gamma_k}{2c} \geq \frac{1}{4}$ and rearranging, we have
\begin{align}
\frac{2\eta_{k+1}^{\max}}{\eta_{k+1}^{\min}}\ec[k]{\sqn{x_{k+1} - x_*}}  + \eta_k^{\max}\gamma_k \br{f(x_k) - f_*}  &\leq \frac{2\eta_k^{\max}}{\eta_k^{\min}}\sqn{x_k - x_*}\\
&+ 4\gamma_k\eta_k^{\max}\br{\frac{\eta_k^{\max}\cL}{2(1-c)} - 1}_+\bar{\sigma}^2 + \frac{2\gamma_k^2\eta_k^{\max}\bar{\sigma}^2}{c},
\end{align}
where for all $a \in \R, (a)_+ = \max(a, 0)$.
\end{proof}

\subsubsection{Proof of Theorem \ref{thm:adap_asymp}}
\begin{proof}
Using the inequality \eqref{eq:rec_adap_step} from Theorem \ref{thm:rec_adap_step}, the proof of Theorem \ref{thm:adap_asymp} procedes exactly as the proof of Theorem \ref{thm:as_conv_sgd}, with the conditions on the stepsizes of Theorem \ref{thm:adap_asymp} instead of Condition \ref{con:step_sizes}. See Section \ref{sec:sgd_asymp} and \ref{sec:app_proofs_as_conv_sgd}.
\end{proof}

\subsubsection{Proof of Theorem \ref{thm:adap_nonasymp}}
\begin{proof}
Taking the expectation in \eqref{eq:rec_adap_step}, rearranging and summing between $t=0,\dots,k-1$, we have, 
\begin{align}
\sum_{t=0}^{k-1}\eta_t^{\max}\gamma_t \ec{f(x_t) - f_*}  &\leq 2a_0\sqn{x_0 - x_*} - 2a_{k+1}\ecn{x_{k+1} - x_*}\\
&+ 4\sum_{t=0}^{k-1}\gamma_t\eta_t^{\max}\br{\frac{\eta_t^{\max}\cL}{2(1-c)} - 1}_+\bar{\sigma}^2 + \frac{2\sum_{t=0}^{k-1}\gamma_t^2\eta_t^{\max}\bar{\sigma}^2}{c}.
\end{align}
Dividing by $\sum_{j=0}^{k-1}\eta_t^{\max}\gamma_t $ and using Jensen's inequality gives the desired result.
\end{proof}

\section{Proofs for Section \ref{sec:shb_asymptotic}} \label{sec:app_proofs_shb}

\subsection{Proof of Theorem \ref{thm:as_conv_shb}}

In the remainder of this section and the forthcoming lemmas we consider the iterates of \eqref{eq:SHB} and the setting of Theorem \ref{thm:as_conv_shb}, that is
\begin{align}
\lambda_0 = 0, \; \lambda_k = \lambda_k = \frac{\sum_{t=0}^{k-1}\eta_t}{4\eta_k}, \; \alpha_k = \frac{\eta_k}{1+\lambda_{k+1}} \; \mbox{and} \; \beta_k = \frac{\lambda_k}{1+\lambda_{k+1}}, \label{eq:lambdaappendix}
\end{align}
where 
\begin{equation}\label{eq:etarestappendix}
 0<\eta_k \leq 1/4\cL, \sum_k \eta_k^2 \sigma^2 < \infty \quad and \quad \sum_k \eta_k = \infty.
\end{equation} 
Note that from \eqref{eq:SHB_IMA}, we have
\begin{align}\label{eq:zkapp}
z_k = x_k + \lambda_k\br{x_k - x_{k-1}}.
\end{align}
We also assume that Assumption~\ref{asm:smoothconvex} holds throughout.

To make the proof more readable, we first state and prove the two following lemmas.

\begin{lemma}\label{lem:conv_eta_diff_f}
$\sum_k \eta_k\br{f(x_k) - f_*} < +\infty \; \mbox{almost surely}.$
\end{lemma}

\begin{lemma}\label{lem:conv_lambda_diff_x}
$\sum_k \lambda_{k+1} \sqn{x_k - x_{k-1}} < +\infty$, and thus, $\lim_k \lambda_{k+1} \sqn{x_{k+1} - x_{k}} = 0 \; \mbox{almost surely}.$
\end{lemma}

We first prove Lemma \ref{lem:conv_eta_diff_f}.
\begin{proof}[Proof of Lemma \ref{lem:conv_eta_diff_f}]
From \eqref{eq:midstep_as_conv_shb}, we have 
\begin{align}
\ec[k]{\sqn{z_{k+1} - x_*}} &\leq \sqn{z_{k} - x_*} - 2\eta_k\br{\frac{1}{2} + \lambda_k}\br{f(x_k) - f_*} \nonumber \\
& + 2\eta_k\lambda_k\br{f(x_{k-1}) - f_*} + 2\eta_k^2\sigma^2. \label{eq:step1_sto}
\end{align}

Using~\eqref{eq:lambdakasymptotic} we have  that
\begin{align} 
 2\eta_k \br{\frac{1}{2} + \lambda_k}  &= 2\eta_k\br{\frac{2\eta_k + \sum_{t=0}^{k-1}\eta_t}{4\eta_k}}\\
 &=2\eta_k\br{\frac{\eta_k}{4\eta_k} + \frac{\sum_{t=0}^{k-1}\eta_t}{4\eta_k}}\\ 
 &= 2\eta_k\br{\frac{1}{4} + \frac{\sum_{t=0}^{k-1}\eta_t}{4\eta_{k+1}} \frac{\eta_{k+1}}{\eta_k}  }\\
 &= \frac{\eta_k}{2} + 2\eta_{k+1}\lambda_{k+1}. \label{eq:tempsjo4josj48jds}
\end{align}
Using~\eqref{eq:tempsjo4josj48jds} in~\eqref{eq:step1_sto} gives
\begin{align}
&  \ec[k]{\sqn{z_{k+1} - x_*} + 2 \eta_{k+1} \lambda_{k+1}\br{f(x_k) - f_*}} + \eta_k\br{f(x_k) - f_*} \\
&\qquad \qquad \leq \sqn{z_k - x_*} + 2\eta_k \lambda_k\br{f(x_{k-1}) - f_*} + 2\eta_k^2\sigma^2.
\end{align}

Hence, applying Lemma \ref{lem:simple_RS} with 
\[V_k = \sqn{z_k - x_*} + 2\eta_k \lambda_k\br{f(x_{k-1}) - f_*} , \; \gamma_k =0, \; U_{k+1} = \eta_k\br{f(x_k) - f_*} \; and \; Z_k = 2\eta_k^2\sigma^2,\]
 we have by~\eqref{eq:etarestappendix} that $$\sum_k \eta_k\br{f(x_k) - f_*} < +\infty \; \textit{almost surely}.$$
\end{proof}

We now turn to prove Lemma \ref{lem:conv_lambda_diff_x}.
\begin{proof}[Proof of Lemma \ref{lem:conv_lambda_diff_x}]
We have,
\begin{align}
\ec[k]{\sqn{x_{k+1} - x_k}} &\overset{\eqref{eq:SHB}}{=} \beta_k^2 \sqn{x_k - x_{k-1}} + \alpha_k^2\sqn{\nabla f_{v_k}(x_k)} - 2\alpha_k \beta_k \langle \nabla f(x_k), x_k - x_{k-1} \rangle.
\end{align}

Multiplying by $\br{1+\lambda_{k+1}}^2$ and using~\eqref{eq:lambdaappendix} we have that $\beta_k = \frac{\lambda_k}{1+\lambda_{k+1}}$ and $\alpha_k = \frac{\eta_k}{1+\lambda_{k+1}}$ and thus
\begin{align}
\br{1+\lambda_{k+1}}^2\ec[k]{\sqn{x_{k+1} - x_k}} &= \lambda_k^2 \sqn{x_k - x_{k-1}} + \eta_k^2\sqn{\nabla f_{v_k}(x_k)} - 2\eta_k \lambda_k \langle \nabla f(x_k), x_k - x_{k-1} \rangle.
\end{align}
Thus using the convexity of $f$ and \eqref{eq:expsmooth}, which follows from Lemma~\ref{lem:smoothconvex}, we have
\begin{align}
\br{1+\lambda_{k+1}}^2\ec[k]{\sqn{x_{k+1} - x_k}} &\leq 
\lambda_k^2 \sqn{x_k - x_{k-1}} +4\eta_k^2 \cL\br{f(x_k) - f_*} + 2\eta_k \lambda_k \br{f(x_{k-1}) - f(x_k)} + 2\eta_k^2\sigma^2 \\
&= \lambda_k^2 \sqn{x_k - x_{k-1}} - 2\eta_k \br{\lambda_k - 2\eta_k \cL}\br{f(x_k) - f_*}\\ 
& \quad + 2\eta_k \lambda_k \br{f(x_{k-1}) - f_*} + 2\eta_k^2\sigma^2.
\end{align}

Re-arranging the above gives,
\begin{align}
& \br{1+\lambda_{k+1}}^2\ec[k]{\sqn{x_{k+1} - x_k}} + 2\eta_k \br{\lambda_k - 2\eta_k \cL}\br{f(x_k) - f_*} \\
& \qquad \qquad  \leq \lambda_k^2 \sqn{x_k - x_{k-1}} + 2\eta_k \lambda_k \br{f(x_{k-1}) - f_*} + 2\eta_k^2\sigma^2. \label{eq:tool_lambda_k^2_sto}
\end{align}
Combining both~\eqref{eq:tool_lambda_k^2_sto} and~\eqref{eq:step1_sto}  we have that
\begin{align}
& \ec[k]{\sqn{z_{k+1} - x_*}} + 4 \eta_k \br{\frac{1}{2} - 2\eta_k\cL + \lambda_k}\br{f(x_k) - f_*} + \br{1+\lambda_{k+1}}^2\ec[k]{\sqn{x_{k+1} - x_k}} \\
& \qquad \qquad \leq \sqn{z_k - x_*} + 4\eta_k \lambda_k\br{f(x_{k-1}) - f_*} + \lambda_k^2 \sqn{x_k - x_{k-1}} + 4\eta_k^2\sigma^2.
\end{align}
Hence, since $\eta_k \leq \frac{1}{8\cL}$,
\begin{align}
& \ec[k]{\sqn{z_{k+1} - x_*}} + 4 \eta_k \br{\frac{1}{4} + \lambda_k}\br{f(x_k) - f_*} + \br{1+\lambda_{k+1}}^2\ec[k]{\sqn{x_{k+1} - x_k}} \\
& \qquad \qquad \leq \sqn{z_k - x_*} + 4\eta_k \lambda_k\br{f(x_{k-1}) - f_*} + \lambda_k^2 \sqn{x_k - x_{k-1}} + 4\eta_k^2\sigma^2.
\end{align}
Using~\eqref{eq:lambdaappendix}, we have $ \eta_k \br{\frac{1}{4} + \lambda_k} = \eta_{k+1}\lambda_{k+1}$. Hence,
\begin{align}
& \sqn{z_{k+1} - x_*} + 4 \eta_{k+1}\lambda_{k+1}\br{f(x_k) - f_*} + \br{1+\lambda_{k+1}}^2\sqn{x_{k+1} - x_k} \\
&\qquad \qquad \leq \sqn{z_k - x_*} + 4\eta_k \lambda_k\br{f(x_{k-1}) - f_*} + \lambda_k^2 \sqn{x_k - x_{k-1}}+ 4\eta_k^2\sigma^2.  \label{eq:step_2_sto}
\end{align}
Hence, noting $V_k \eqdef \sqn{z_k - x_*} + 4\eta_k \lambda_k\br{f(x_{k-1}) - f_*} + \lambda_k^2 \sqn{x_k - x_{k-1}}$, we have
\begin{align}
\ec[k]{V_{k+1}} + \br{2\lambda_{k+1}+1}\sqn{x_k - x_{k-1}} \leq V_k + 4\eta_k^2\sigma^2. 
\end{align}
Hence, since $\sum_k \eta_k^2\sigma^2 < +\infty$, applying Lemma \ref{lem:simple_RS}, we have 

\begin{align}
\sum_k \lambda_{k+1} \sqn{x_k - x_{k-1}} < +\infty \quad \textit{almost surely} \mbox{,   thus} \quad  \lim_k \lambda_{k+1} \sqn{x_{k+1} - x_{k}} = 0 \quad \textit{almost surely}.
\end{align}
\end{proof}

We can now prove Theorem \ref{thm:as_conv_shb}.
\begin{proof}[Proof of Theorem \ref{thm:as_conv_shb}]
This proof aims at proving that, \textit{almost surely}
\begin{enumerate}
\item $x_k \underset{k \rightarrow +\infty}{\rightarrow} x_*$ for some $x_* \in \X_*$.
\item $f(x_k) - f_* = o\br{\frac{1}{\sum_{t=0}^{k-1} \eta_t}}$.
\end{enumerate}
In our road to prove the first point, we will prove the second point as a byproduct.

We will now prove that $\lim_k \sqn{z_k - x_*}$ exists \textit{almost surely}.
\begin{eqnarray*}
\sqn{z_k - x_*} &\overset{\eqref{eq:zkapp}}{ =} & \sqn{x_k - x_* + \lambda_k\br{x_k - x_{k-1}}}\\
&=& \lambda_k^2\sqn{x_k - x_{k-1}} + 2\lambda_k\langle x_k - x_*, x_k - x_{k-1}\rangle + \sqn{x_k - x_*}\\
&=& \br{\lambda_k^2 + \lambda_k}\sqn{x_k - x_{k-1}} + \lambda_k \br{\sqn{x_k - x_*} - \sqn{x_{k-1} - x_*}} + \sqn{x_k - x_*}.
\end{eqnarray*}
Define 
\begin{align} \label{eq:delta_kapp}
\delta_k \eqdef \lambda_k \br{\sqn{x_k - x_*} - \sqn{x_{k-1} - x_*}} + \sqn{x_k - x_*},
\end{align}
so that
\begin{align} \label{eq:iterateconvapp}
\sqn{z_k - x_*} = \br{\lambda_k^2 + \lambda_k}\sqn{x_k - x_{k-1}} + \delta_k.
\end{align}
We will first prove that $\lim_k \br{\lambda_k^2 + \lambda_k}\sqn{x_k - x_{k-1}}$ exists \textit{almost surely}, then that $\lim_k \delta_k$ exists \textit{almost surely}.

\  First, we have from Lemma \ref{lem:conv_lambda_diff_x} that $(\lambda_k \sqn{x_k - x_{k-1}})_k$ converges to zero \textit{almost surely}. Hence, it remains to show that $\lim_k \lambda_k^2 \sqn{x_k - x_{k-1}}$ exists \textit{almost surely}.
From \eqref{eq:tool_lambda_k^2_sto}, we have that 
\begin{align}
&\lambda_{k+1}^2\ec[k]{\sqn{x_{k+1} - x_k}} + 2\eta_k\br{\lambda_k - 2\eta_k\cL}\br{f(x_k) - f_*}\\
&\qquad \qquad \leq \lambda_k^2 \sqn{x_k - x_{k-1}} + 2\eta_k \lambda_k \br{f(x_{k-1}) - f_*} + 2\eta_k^2\sigma^2. 
\end{align}
Using \eqref{eq:lambdaappendix} and the fact that $\eta_k \leq \frac{1}{8\cL}$, we have that $2\eta_{k+1}\lambda_{k+1} = 2\eta_k\br{\frac{1}{4} + \lambda_k} \leq 2\eta_k\br{\frac{1}{2} - 2\eta_k \cL + \lambda_k}$. Hence, $2\eta_k\br{\lambda_k - 2\eta_l\cL} \geq 2\eta_{k+1}\lambda_{k+1} - \eta_k$. Therefore, denoting 
\begin{align}
d_k \eqdef \sqn{x_k - x_{k-1}} \; \mbox{and} \; \theta_k \eqdef 2\eta_k\br{f(x_{k-1}) - f_*},
\end{align}
we have
\begin{align}
\ec[k]{\lambda_{k+1}^2d_k + \lambda_{k+1}\theta_{k+1}} \leq \lambda_{k}^2d_k + \lambda_{k}\theta_{k} + \eta_k\br{f(x_k) - f_*} + 2\eta_k^2\sigma^2.
\end{align}

From Lemma \ref{lem:conv_eta_diff_f}, we have $\sum_k \eta_k\br{f(x_k) - f_*} < +\infty$. Moreover, $\sum_k \eta_k^2 \sigma^2 < +\infty$. Hence, we have by Lemma \ref{lem:simple_RS} that $\lim_k \lambda_{k}^2d_k + \lambda_{k}\theta_{k}$ exists \textit{almost surely}.

Moreover, by Lemma \ref{lem:conv_lambda_diff_x}, $\sum_k \lambda_k d_k < + \infty$, and we have $\sum_k \theta_k < +\infty$ \textit{almost surely}. Hence, ${\sum_k \lambda_kd_k + \theta_k < +\infty}$ \textit{a.s}. Rewriting 
\begin{align}
\lambda_kd_k + \theta_k = \frac{1}{\lambda_k}\br{\lambda_k^2d_k + \lambda_k\theta_k},
\end{align}
we have, since $\lim_k \lambda_{k}^2d_k + \lambda_{k}\theta_{k}$ exists \textit{almost surely} and 
\begin{eqnarray*}
\sum_k \frac{1}{\lambda_k} & \overset{0<\eta_k \leq \frac{1}{4\cL}}{ =} &4\sum_k\frac{\eta_k}{\sum_{t=0}^{k-1}\eta_t} \; \overset{\mbox{Condition } \ref{con:step_sizes}}{=} \; \infty
\end{eqnarray*}
that
\begin{align}
\lim_k \lambda_{k}^2d_k + \lambda_{k}\theta_{k} = 0,
\end{align}
which means that both $\lim_k \lambda_{k}^2d_k = 0$ and $\lim_k \lambda_{k}\theta_{k} = 0$ \textit{a.s}. Writing out $\lim_k \lambda_{k}\theta_{k} = 0$ explicitly, we have 
\begin{align}
\boxed{f(x_k) - f_* = o\br{\frac{1}{\sum_{t=0}^{k-1}\eta_t}} \; \textit{almost surely}}
\end{align}
This proves the second point of Theorem \ref{thm:as_conv_shb} and that $\lim_k \lambda_k^2d_k = 0 \; \textit{almost surely}$.
To prove the first point of Theorem \ref{thm:as_conv_shb} 
it remains to show that $\lim_k \delta_k$ exists \textit{almost surely}.

Note $u_k = \sqn{x_k - x_*}$. We have
\begin{eqnarray}
u_{k+1} &\overset{\eqref{eq:SHB}}{ =}& \sqn{x_k - x_* + \beta_k \br{x_k - x_{k-1}}} + \alpha_k^2 \sqn{\nabla f_{v_k}(x_k)} - 2\alpha_k\beta_k\langle \nabla f_{v_k}(x_k), x_k - x_{k-1} \rangle \\
&& - 2\alpha_k \langle \nabla f_{v_k}(x_k), x_k - x_* \rangle. \nonumber
\end{eqnarray}
Taking expectation conditioned on $x_k$, and using the convexity of $f$ and Lemma~\ref{lem:smoothconvex} we have that
\begin{align}
\ec[k]{u_{k+1}}&\leq  \sqn{x_k - x_* + \beta_k \br{x_k - x_{k-1}}} - 2\alpha_k\br{1 + \beta_k - 2\alpha_k \cL}\br{f(x_k) - f_*}\\
& + 2\alpha_k \beta_k\br{f(x_{k-1}) - f_*} + 2\alpha_k^2\sigma^2. \label{eq:tempao38j9aj3}
\end{align}
Furthermore note that
\begin{align}
\sqn{x_k - x_* + \beta_k \br{x_k - x_{k-1}}} &= u_k + \beta_k^2 d_k + 2\beta_k \langle x_k - x_*, x_{k} - x_{k-1}\rangle \\
&= u_k + \br{\beta_k^2 + \beta_k}d_k + \beta_k\br{u_k - u_{k-1}}. \label{eq:tempao3jaj3a}
\end{align}

Hence, using the fact that $0 \leq \beta_k \leq 1$ and inserting~\eqref{eq:tempao3jaj3a} into~\eqref{eq:tempao38j9aj3} gives
\begin{align}
\ec[k]{u_{k+1}} &\leq u_k + 2 d_k + \beta_k\br{u_k - u_{k-1}} - 2\alpha_k\br{1 + \beta_k - 2\alpha_k \cL}\br{f(x_k) - f_*} \\
& + 2\alpha_k \beta_k\br{f(x_{k-1}) - f_*} + 2\alpha_k^2\sigma^2.
\end{align}
Multiplying the above by $(1 + \lambda_{k+1})$, rearranging and using~\eqref{eq:alphafrometalambdaintro} results in
\begin{align}
(1+\lambda_{k+1})\ec[k]{u_{k+1} - u_k}  &\leq 2 \br{1+\lambda_{k+1}} d_k + \lambda_k\br{u_k - u_{k-1}} - 2\eta_k\br{1 + \beta_k - 2\alpha_k \cL}\br{f(x_k) - f_*} \\
& + 2\eta_k \beta_k\br{f(x_{k-1}) - f_*} + 2\frac{\eta_k^2\sigma^2}{1 + \lambda_{k+1}}. \label{eq:tempjao83jaaj93}
\end{align}
 Using the definition of $\delta_{k+1}$ given in~\eqref{eq:delta_kapp}  we have that
 $$\delta_{k+1} - \delta_k = \br{1+\lambda_{k+1}}\br{u_{k+1} - u_k} - \lambda_k \br{u_k - u_{k-1}},$$ 
which we use to re-write~\eqref{eq:tempjao83jaaj93} as 
 we have
\begin{align}
	 \delta_{k+1} +\frac{\eta_k}{\eta_{k+1}}\br{1 + \beta_k - 2\alpha_k \cL}\theta_{k+1}  \leq \delta_k +2 \br{1+\lambda_{k+1}} d_k  + \beta_k\theta_k + 2\frac{\eta_k^2\sigma^2}{1 + \lambda_{k+1}}.
\end{align}
Hence, since $\eta_{k+1} \leq \eta_k$,
\begin{align}
\ec[k]{\delta_{k+1} + \br{1+\beta_k - 2\alpha_k\cL}\theta_{k+1}} \leq \delta_{k} + \beta_k\theta_{k} + 2\br{1+\lambda_{k+1}}d_k + 2\frac{\eta_k^2}{1+\lambda_{k+1}}\sigma^2.
\end{align}
And since,
\begin{align}
1 + \beta_k - 2\alpha_k\cL = 1 + \frac{\lambda_k}{1+ \lambda_{k+1}} - \frac{2\eta_k\cL}{1+ \lambda_{k+1}} &= \frac{1}{1+ \lambda_{k+1}}\br{1 + \lambda_{k+1} + \lambda_k - 2\eta_k\cL}\\
&\geq \frac{\lambda_{k+1}}{1+ \lambda_{k+1}} \geq \frac{\lambda_{k+1}}{1+ \lambda_{k+2}} = \beta_{k+1},
\end{align}
we have
\begin{align}
\ec[k]{\delta_{k+1} + \beta_{k+1}\theta_{k+1}} \leq \br{\delta_{k} + \beta_k\theta_{k}} + 2\br{1+\lambda_{k+1}}d_k + \frac{2\eta_k^2}{1+\lambda_{k+1}}\sigma^2.
\end{align}
Since by Lemma \ref{lem:conv_lambda_diff_x} we have that $\sum_k 2\br{1+\lambda_{k+1}}d_k < +\infty \; \textit{almost surely}$, and $\sum_k \frac{\eta_k^2\sigma^2}{1+\lambda_{k+1}} < +\infty$, we have by Lemma \ref{lem:simple_RS} that $\lim_k \delta_{k} + \beta_k\theta_{k} $ exists \textit{almost surely} And since $\lim_k \beta_k \theta_k = 0 \; \textit{almost surely}$, we deduce that $\lim_k \delta_k$ exists \textit{almost surely}.

Thus we have now shown that $\lim_k \sqn{z_k - x_*}$ exists \textit{almost surely}. Therefore, since $x_k - x_* = z_k - x_* - \lambda_k\br{x_k - x_{k-1}}$ and
\begin{align}
|\norm{x_k - x_*} - \norm{z_k - x_*}| \leq \lambda_k\norm{x_k - x_{k-1}} \underset{k \rightarrow +\infty}{\rightarrow} 0\quad  \textit{almost surely},
\end{align}
we have that $\lim_k \norm{x_k - x_*} - \norm{z_k - x_*}$ exists \textit{almost surely}, and so does $\lim_k \norm{x_k - x_*}$.

We also have that both $\norm{z_k - x_*}$ and $\lambda_k\norm{x_k - x_{k-1}}$ are bounded \textit{almost surely}, thus $\norm{x_k - x_*}$ is bounded \textit{almost surely} Hence, $(x_k)_k$ is bounded \textit{almost surely}, thus \textit{almost surely} sequentially compact. 

Let $\br{x_{n_k}}_k$ be a subsequence of $\br{x_n}_n$ which converges to some $x \in \R^d \; a.s$. Since $f(x_n) \rightarrow_n f_* \; \textit{almost surely}$  for all $x^* \in \argmin f$, we have $x \in \argmin f \; a.s$. Finally, applying Lemma 2.39 in \cite{Bauschke_convex} (restricted to our finite dimensional setting, where weak convergence and strong convergence are equivalent), there exists $x_* \in \argmin f$ such that
\begin{align}
\boxed{x_k \underset{k \rightarrow +\infty}{\rightarrow} x_*, \quad \textit{almost surely}}
\end{align}

This proves the first point of Theorem \ref{thm:as_conv_shb}.
\end{proof}

\section{Proofs for Section \ref{sec:sgd_nonconvex}} \label{sec:app_proofs_as_conv_sgd_nonconvex}
\subsection{Proof of Theorem \ref{thm:as_conv_sgd_nonconvex}}
\begin{proof}
Consider the setting of Theorem \ref{thm:as_conv_sgd_nonconvex}. Let $\eta_k \leq \frac{1}{LB}$ for all $k \in \N$. From \cite[Proof of Lemma 2]{Khaled20}, we have
\begin{align}
\ec[k]{f(x_{k+1}) - f_*} + \eta_k \sqn{\nabla f(x_k)} \leq \br{1+\eta_k^2 AL}\br{f(x_{k}) - f_*} + \frac{\eta_k^2 LC}{2}.
\end{align}
Since $\sum_k \eta_k^2 < \infty$, we also have that $\prod_{k=0}^\infty (1+\eta_k^2 AL) < \infty$. Thus, by Lemma \ref{lem:simple_RS}, we have that $\br{f(x_{k}) - f_*}_k$ converges \textit{almost surely}.

Define for all $k \in \N$,
\begin{align}
w_k = \frac{2\eta_k}{\sum_{j=0}^k \eta_j}, \quad g_0 = \sqn{\nabla f(x_0)}, \quad g_{k+1} = (1 - w_k) g_k + w_k \sqn{\nabla f(x_k)}.
\end{align}
Note that since $\br{\eta_k}_k$ is decreasing, $w_k \in [0, 1]$. Plugging this back in the previous inequality gives
\begin{align}
\ec[k]{f(x_{k+1}) - f_*} + \frac{\sum_{j=0}^k\eta_j}{2}g_{k+1} + \frac{\eta_k}{2}g_k  \leq \br{1+\eta_k^2 AL}\br{f(x_{k}) - f_*} + \frac{\sum_{j=0}^{k-1}\eta_j}{2}g_{k} + \frac{\eta_k^2 LC}{2}.
\end{align}
Since $\sum_k \eta_k^2 < \infty$, we also have that $\prod_{k=0}^\infty (1+\eta_k^2 AL) < \infty$. Thus, by Lemma \ref{lem:simple_RS}, we have
\begin{align}
&\br{f(x_k) - f_* + \br{\sum_{j=0}^{k-1}\eta_j}g_{k}}_k \; \mbox{converge \textit{almost surely}, and} \; \sum_k \eta_k g_k < \infty \mbox{ \textit{almost surely}}
\end{align}
And since $\br{f(x_k) - f_*}_k$ converges \textit{almost surely}, we have that $\br{\br{\sum_{j=0}^{k-1}\eta_j}g_{k}}_k$ converges \textit{almost surely}.
Hence, we have that $\lim_k \frac{\eta_k}{\sum_{j=0}^{k-1} \eta_j} \sum_{j=0}^{k-1} \eta_j g_k = \lim_{k} \eta_k g_k = 0$. But since we assumed that $\sum_k \frac{\eta_k}{\sum_{j=0}^{k-1} \eta_j}$ diverges, this implies that $\lim_k \sum_{j=0}^{k-1}\eta_j g_k = 0$, that is, we have that,
\begin{align}
g_k = o\br{\frac{1}{\sum_{j=0}^{k-1}\eta_j}} \mbox{ almost surely}
\end{align}
But since for all $k\in \N, g_{k+1} = (1 - w_k) g_k + w_k \sqn{\nabla f(x_k)}$, $g_k$ is a weighted average of all past $\sqn{\nabla f(x_j)}, j=0,\dots,k-1$. Hence, there exists a sequence $\br{\tilde{w}_j}_j$ in $[0, 1]$ which verifies $\sum_{j=0}^{k-1}\tilde{w}_j = 1$  such that $g_k = \sum_{j=0}^{k-1} \tilde{w}_j \sqn{\nabla f(x_j)}$. Thus, $g_k \geq \min_{t=0,\dots,k-1} \sqn{\nabla f(x_t)} \geq 0$. Hence we have \emph{almost surely}
\begin{align}
\boxed{ \min_{t=0,\dots,k-1} \sqn{\nabla f(x_t)} = o\br{\frac{1}{\sum_{j=0}^{k-1}\eta_j}}}
\end{align}

\end{proof}

\section{Extension of our results to the nonsmooth setting}\label{sec:app_nonsmooth}\label{sec:app_extension_nonsmooth}
In this section, we will consider the stochastic subgradient descent method under the bounded gradients assumption, as in \cite{Nemirovski09} . Under this assumption, we show that we can derive the same convergence rates as in Theorem \ref{thm:as_conv_sgd}.

\begin{proposition}
Consider the following method: at each iteration $k$, let $g_k$ be such that $\ec[k]{g_k} = g(x_k)$ for some $g(x_k) \in \partial f(x_k)$, and update
\begin{align}
x_{k+1} = x_k - \eta_k g_k,
\end{align}
where we assume that $f$ is convex and that there exists $G$ such that $\forall k \in \N, \; \ecn{g_k} \leq G$.

Choose step sizes $\br{\eta_k}_k$ which verify Condition \ref{con:step_sizes} (with $G$ in place of $\sigma^2$). Define for all $k \in \N$
\begin{eqnarray}
w_k = \frac{2\eta_k}{\sum_{j=0}^{k}\eta_j} \quad \mbox{and} \quad \left\{
            \begin{array}{l}
            		\bar{x}_0 = x_0 \\
            		\bar{x}_{k+1} = w_k x_k + (1 - w_k) \bar{x}_k.
            \end{array} \right.  \label{eq:w_xbar_ns}
\end{eqnarray}
Then, we have $a.s$. that $f(\bar{x}_k) - f_* = o\br{\frac{1}{\sum_{t=0}^{k-1} \eta_t}}.$
\end{proposition}

\begin{proof}
The proof procedes exactly as in the smooth case, but with replacing the bound \eqref{eq:expsmooth} by the bound $\forall k \in \N, \; \ecn{g_k} \leq G$. Indeed, expanding the squares we have that
\begin{align}
\sqn{x_{k+1} - x_*} = \sqn{x_{k} - x_*} - 2 \eta_k \langle g_k, x_k - x_* \rangle + \eta_k^2\sqn{g_k}.
\end{align}
Then taking conditional expectation $\ec[k]{\cdot} \eqdef \ec{\cdot \; | \; x_k}$ gives, since $\ec[k]{g_k} = g(x_k)$ for some $g(x_k) \in \partial f(x_k)$, we have
\begin{eqnarray*}
\ec[k]{\sqn{x_{k+1} - x_*}} &=& \sqn{x_{k} - x_*} - 2 \eta_k \langle g(x_k), x_k - x_* \rangle + \eta_k^2\ec[k]{\sqn{g_k}} \\
&\leq& \sqn{x_{k} - x_*} - 2 \eta_k\br{f(x_k) - f_*} + \eta_k^2G,
\end{eqnarray*}
where we used in the last inequality the fact that $g(x_k)$ is a subgradient of $f$ at $x_k$, and that $\ec[k]{\sqn{g_k}} \leq G$.
Hence, rearranging, we have
\begin{align}\label{eq:main_rec_sgd_ns}
\ec[k]{\sqn{x_{k+1} - x_*}} + 2\eta_k \br{f(x_k) - f_*} \leq \sqn{x_{k} - x_*} + \eta_k^2 G.
\end{align}

From~\eqref{eq:w_xbar_ns} we have that $w_k = \frac{2\eta_k}{\sum_{j=0}^{k}\eta_j}$. Since $w_0 = 2\frac{\eta_0}{\eta_0}=2$ we have that $\bar{x}_1 = 2 x_0 - \bar{x}_0 = 2x_0 - x_0 = x_0$. Hence, it holds that
\begin{align}\label{eq:init_verification_weights_ns}
f(\bar{x}_1) - f_* = f(x_0) - f_* = w_0 \br{f(x_0) - f_*} + (1 - w_0)\br{f(\bar{x}_0) - f_*} .
\end{align}

Now for $k \in \N^*$ we have that following equivalence 
\[w_k \in [0,1] \quad  \iff \quad 2\eta_k \leq \sum_{j=0}^{k} \eta_j  \quad  \iff \quad \eta_k \leq \sum_{j=0}^{k-1}\eta_j.\]

The right hand side of the equivalence holds because $\br{\eta_k}_k$ is a decreasing sequence.
 Hence, by Jensen's inequality, we have $\forall k \in \N^*,$
\begin{align}
f(\bar{x}_{k+1}) - f_* \leq w_k \br{f(x_k) - f_*} + (1-w_k) \br{f(\bar{x}_k) - f_*}.
\end{align}

Together with \eqref{eq:init_verification_weights_ns}, this shows that the last inequality holds for all $k \in \N$. Thus,
\begin{align}
\eta_k\br{f(x_k) - f_*} \geq \frac{\eta_k}{w_k}\br{f(\bar{x}_{k+1}) - f_*} - \eta_k\br{\frac{1}{w_k} - 1}\br{f(\bar{x}_{k}) - f_*}.
\end{align}
Replacing this expression in \eqref{eq:main_rec_sgd_ns} gives:
\begin{align}
&\ec[k]{\sqn{x_{k+1} - x_*}} +  \frac{\eta_k}{w_k}\br{f(\bar{x}_{k+1}) - f_*}\\
 &\quad \quad \leq \sqn{x_{k} - x_*} + \eta_k\br{\frac{1}{w_k} - 1}\br{f(\bar{x}_{k}) - f_*}
+  \eta_k^2G.
\end{align}
 Hence substituting in the definition of $w_k$ from~\eqref{eq:w_xbar_ns} gives
\begin{align}
&\ec[k]{\sqn{x_{k+1} - x_*}} +  \sum_{j=0}^{k}\eta_j\br{f(\bar{x}_{k+1}) - f_*}\\
 &\quad\quad\leq \sqn{x_{k} - x_*} + \br{\sum_{j=0}^{k-1}\eta_j - \eta_k}\br{f(\bar{x}_{k}) - f_*}
+ \eta_k^2 G.
\end{align}
Thus re-arranging
\begin{align}
&\ec[k]{\sqn{x_{k+1} - x_*}} + \sum_{j=0}^{k}\eta_j\br{f(\bar{x}_{k+1}) - f_*} + \eta_k\br{f(\bar{x}_{k}) - f_*} \\
 &\quad\quad\leq \sqn{x_{k} - x_*} + \sum_{j=0}^{k-1}\eta_j\br{f(\bar{x}_{k}) - f_*}
+  \eta_k^2G,
\end{align}
which, by Lemma \ref{lem:simple_RS}, has the three following consequences:
\begin{align}
&(\sqn{x_k - x_*})_k \; \mbox{and } \br{\sum_{j=0}^{k}\eta_j\br{f(\bar{x}_{k+1}) - f_*}}_k \; \mbox{converge \textit{almost surely}},\\
& \mbox{and} \; \sum_k \eta_k\br{f(\bar{x}_k) - f_*} < \infty.
\end{align}
Hence, we have that $\lim_k \frac{\eta_k}{\sum_{j=0}^{k-1} \eta_j} \sum_{j=0}^{k-1} \eta_j\br{f(\bar{x}_k) - f_*} = \lim_{k} \eta_k \br{f(\bar{x}_k) - f_*} = 0$. But since we assumed that $\sum_k \frac{\eta_k}{\sum_{j=0}^{k-1} \eta_j}$ diverges, this implies that $\lim_k \sum_{j=0}^{k-1}\eta_j\br{f(\bar{x}_{k+1}) - f_*} = 0$, that is
\begin{align}
\boxed{f(\bar{x}_k) - f_* = o\br{\frac{1}{\sum_{j=0}^{k-1}\eta_j}}}
\end{align}
\end{proof}

\section{Convergence rates for SHB in expectation without the bounded gradients and bounded gradient variance assumptions}\label{sec:app_shb_exp}

Our first theorem provides a non-asymptotic upper bound on the suboptimality given any sequence of step sizes. Later we develop special cases of this theorem through different choices of the stepsizes.

\begin{theorem}\label{theo:shb_convex}
Let Assumption \ref{asm:smoothconvex} hold. Let  $x_{-1} = x_0$. Consider the iterates of~\ref{eq:SHB_IMA}.  
Let $(\eta_k)_k$ be such that $0< \eta_k \leq \frac{1}{4\cL}$ for all $k \in \N$. Define 
$\lambda_0 \eqdef 0$ and $\lambda_{k} = \frac{\sum_{t=0}^{k-1}\eta_t}{2\eta_{k}} \text{   for } k \geq 1.$ Then,
\begin{align}
\ec{f(x_k) - f_*} \leq \frac{\sqn{x_0 - x^*}}{\sum_{t=0}^{k}\eta_t } + 2\sigma^2\frac{\sum_{t=0}^{k}\eta_t^2}{\sum_{t=0}^{k}\eta_t}. \label{eq:master_eq_convex}
\end{align}
\end{theorem}
Note that in Theorem~\ref{theo:shb_convex} the only free parameters are the $\eta_k$'s which in the iterate-moving-average viewpoint~\eqref{eq:SHB_IMA} play the role of a learning rate.  The scaled step sizes $\alpha_k$ and the momentum parameters $\beta_k$ of the usual formulation \eqref{eq:SHB} are given by \eqref{eq:alphafrometalambdaintro} once we have chosen $\eta_k$. We now explore three different settings of the $\eta_k$'s in the following corollaries.

\begin{corollary}
\label{cor:convneigh}
Consider the setting of Theorem \ref{theo:shb_convex}. Let $\eta \leq 1/4\cL$.
\begin{enumerate}
\Item  Let $\eta_k = \eta$. Then, \inlineequation[eq:conv_neigh]{\ec{f(x_k) - f_*} \leq \frac{\sqn{x_0 - x_*}}{\eta \br{k+1}} + 2\eta \sigma^2. \hfill}
\vskip\baselineskip
\item Let $\eta_k = \frac{\eta}{\sqrt{k+1}}$. Then,
\inlineequation[eq:slowconv_shb]{\ec{f(x_{k}) - f_*} \leq\frac{\norm{x^0 - x^*}_2^2  + 4\sigma^2 \eta^2\left(\log(k+1) + 1 \right)}{2\eta\br{\sqrt{k+1} -1}}  \sim \; O\br{\frac{\log(k)}{\sqrt{k}}}. \hfill}
\item Suppose Algorithm \eqref{eq:SHB} is run for $T$ iterations. Let $\eta_k = \frac{\eta}{\sqrt{T+1}}$ for all $k \in \left\{0,\dots,T\right\}$. Then,
\begin{eqnarray}\label{eq:bound_SHB_convex_dec}
\ec{f(x_T) - f_*} \leq
\frac{\norm{x^0 - x^*}_2^2  + 2\sigma^2 \eta^2}{\eta\sqrt{T+1}}.
\end{eqnarray}
\end{enumerate}
\end{corollary}

\eqref{eq:conv_neigh} shows how to set the parameters of SHB so that the last iterate converges sublinearly to a neighborhood of the minimum. In particular, for overparametrized models with $\sigma^2 = 0$, the last iterate of SHB converges sublinearly to the minimum. Moreover, when using the full gradient, which corresponds to directly using the gradient $\nabla f(x_k)$ at each iteration, we have $\cL = L$ and $\sigma^2 = 0$, which recovers the rate derived in \cite{Ghadimi2014} for the deterministic HB method upto a constant.

The $O\br{\log(k)/\sqrt{k}}$ convergence rate in \eqref{eq:slowconv_shb} is the same rate that can be derived for the iterates of SGD, as is done by \cite{Nemirovski09} for a weighted average of the iterates of SGD, or by \cite{orabona2020last} for the last iterate. The difference with SGD is that it is also possible to drop the $\log(k)$ factor in \eqref{eq:slowconv_shb} for the last iterate if we know the stopping time of the algorithm as shown in \eqref{eq:bound_SHB_convex_dec}. So far in the litterature, shaving of this log factor has been shown only for convex lipschitz functions over closed bounded sets \citep{Jain19}.

\subsection{Proof of Theorem~\ref{theo:shb_convex} }

The proof uses the following Lyaponuv function
\[L_k = \ecn{z_k - x_*}  + 2\eta_k\lambda_k\ec{f(x_{k-1}) - f_*}\]
\begin{proof}
We have
\begin{align}
\sqn{z_{k+1} - x_*} &= \sqn{z_{k} - x_* - \eta_k \nabla f_{v_k}(x_k)} \nonumber \\
&\overset{\eqref{eq:SHB_IMA}}{=} \sqn{z_{k} - x_*} - 2 \eta_k \langle \nabla f_{v_k}(x_k), z_k - x_* \rangle   + \eta_k^2 \sqn{\nabla f_{v_k}(x_k) }\nonumber \\
&\overset{\eqref{eq:SHB_IMA}}{=} \sqn{z_{k} - x_*} - 2 \eta_k \langle \nabla f_{v_k}(x_k), x_k - x_* \rangle - 2 \eta_k\lambda_k \langle \nabla f_{v_k}(x_k), x_k - x_{k-1} \rangle +  \eta_k^2 \sqn{\nabla f_{v_k}(x_k) }\nonumber
\end{align}

Then taking conditional expectation $\ec[k]{\cdot} \eqdef \ec{\cdot \; | \; x_k}$ we have
\begin{align}
\ec[k]{\sqn{z_{k+1} - x_*}} &= \sqn{z_{k} - x_*}  - 2\eta_k \langle \nabla f(x_k) , x_k - x_* \rangle \nonumber \\
& - 2\eta_k\lambda_k  \langle \nabla f(x_k) , x_k - x_{k-1} \rangle + \eta_k^2 \ec[k]{\sqn{\nabla f_{v_k}(x_k)}} , \nonumber \\
&\overset{\eqref{eq:expsmooth}+\eqref{eq:convexity}}{\leq} A_k + 4\eta_k^2\cL\br{f(x_k) - f_*} + 2 \eta_k^2 \sigma^2 \nonumber \\
& - 2\eta_k \br{f(x_k) - f_*)} - 2 \eta_k\lambda_k \br{f(x_k) - f(x_{k-1})} \nonumber \\
&=  \sqn{z_{k} - x_*} - 2\eta_k\br{1 + \lambda_k - 2\eta_k \cL}\br{f(x_k) - f_*} \nonumber \\
& + 2\eta_k\lambda_k\br{f(x_{k-1}) - f_*} + 2\eta_k^2\sigma^2 .\label{eq:aj9e8j38a} \\
&\leq \sqn{z_{k} - x_*} - 2\eta_k\br{\frac{1}{2} + \lambda_k}\br{f(x_k) - f_*} \nonumber \\
& + 2\eta_k\lambda_k\br{f(x_{k-1}) - f_*} + 2\eta_k^2\sigma^2 , \label{eq:midstep_as_conv_shb}
\end{align}
where we used the fact that $\eta_k \leq \frac{1}{4\cL}$ in the last inequality.
Since $\lambda_{k+1} = \frac{\sum_{t=0}^{k}\eta_t}{2\eta_{k+1}}$ we have  that 
\[\eta_{k+1}\lambda_{k+1} =\eta_k\br{\frac{1}{2}+\lambda_k}. \]
Using this in~\eqref{eq:aj9e8j38a} then taking expectation and rearranging gives
\begin{eqnarray}
\ec{\sqn{z_{k+1} - x_*}} + 2\eta_{k+1}\lambda_{k+1}\ec{f(x_k) - f_*}
\leq \ec{\sqn{z_{k} - x_*}}  + 2\eta_k\lambda_k\ec{f(x_{k-1}) - f_*} + 2\eta_k^2\sigma^2. \nonumber
\end{eqnarray}
Summing over $t=0$ to $k$ and using a telescopic sum, we have
\begin{eqnarray}
\ec{\sqn{z_{k+1} - x_*}} + \br{\sum_{t=0}^{k} \eta_t} \ec{f(x_k) - f_*}
\leq \sqn{x_0 - x^*} + 2\sigma^2 \sum_{t=0}^{k}\eta_t^2,\nonumber
\end{eqnarray}
where we used that $\lambda_0 =0.$
Thus, writing $\lambda_k$ explicitly, gives
\begin{eqnarray}
\ec{f(x_k) - f_*}
\leq \frac{\sqn{x_0 - x^*}}{\sum_{t=0}^{k}\eta_t} + \frac{2\sigma^2\sum_{t=0}^{k}\eta_t^2}{\sum_{t=0}^{k}\eta_t}. \nonumber
\end{eqnarray}
\end{proof}

\subsection{Proof of Corollary~\ref{cor:convneigh}}

\begin{proof}
\eqref{eq:conv_neigh} and \eqref{eq:bound_SHB_convex_dec} can be easily derived from Theorem \ref{theo:shb_convex}. \eqref{eq:slowconv_shb} requires some additional sum computations. Using the integral bound and plugging in our choice of $\eta_k$ gives 
\begin{equation}
\sum_{t=0}^{k-1}\eta_t^2  = \eta^2 \sum_{t=0}^{k-1} \frac{1}{t+1} \; \leq \; \eta^2\left(\log(k) + 1 \right)  \quad \mbox{and} \quad  \sum_{t=0}^{k-1}\eta_t \geq 2\eta\left(\sqrt{k} -1\right), \label{eq:ks94oo8s84}
\end{equation}
which we use to obtain \eqref{eq:slowconv_shb}.
\end{proof}

\end{appendices}

\end{document}

%% file: header.tex
\newcommand\ec[2][]{\ensuremath{\mathbb{E}_{#1} \left[#2\right]}}

\newcommand\ecn[2][]{\ec[#1]{\norm{#2}^2}}
\newcommand\ecd[2][]{\ensuremath{\mathbb{E}_{\mathcal{D}} \left[#2\right]}}

\newcommand\br[1]{\left ( #1 \right )}

\newcommand{\N}{\mathbb{N}}

\newcommand{\R}{\mathbb{R}}
\newcommand{\X}{\mathcal{X}}

\newcommand{\F}{\mathcal{F}}

\DeclareMathOperator*{\argmin}{\arg\!\min}

\newcommand{\norm}[1]{\left\lVert#1\right\rVert}
\newcommand{\sqn}[1]{{\left\lVert#1\right\rVert}^2}

\newcommand{\D}{\mathcal{D}}

\newcommand{\eqdef}{\overset{\text{def}}{=}}
\makeatletter
\newsavebox\myboxA
\newsavebox\myboxB
\newlength\mylenA

\newcommand{\cX}{\mathcal{X}}

\newcommand{\cL}{\mathcal{L}}

\usepackage{tcolorbox}
\usepackage{pifont}
\definecolor{mydarkgreen}{RGB}{39,130,67}

\definecolor{mydarkred}{RGB}{192,47,25}

\newcommand*\overbar[2][0.75]{%
    \sbox{\myboxA}{$\m@th#2$}%
    \setbox\myboxB\null
    \ht\myboxB=\ht\myboxA%
    \dp\myboxB=\dp\myboxA%
    \wd\myboxB=#1\wd\myboxA
    \sbox\myboxB{$\m@th\overline{\copy\myboxB}$}
    \setlength\mylenA{\the\wd\myboxA}
    \addtolength\mylenA{-\the\wd\myboxB}%
    \ifdim\wd\myboxB<\wd\myboxA%
       \rlap{\hskip 0.5\mylenA\usebox\myboxB}{\usebox\myboxA}%
    \else
        \hskip -0.5\mylenA\rlap{\usebox\myboxA}{\hskip 0.5\mylenA\usebox\myboxB}%
    \fi}
\makeatother

\usepackage[colorinlistoftodos,bordercolor=orange,backgroundcolor=orange!20,linecolor=orange,textsize=scriptsize]{todonotes}

\newcommand{\rob}[1]{}

\usepackage{mdframed} 
\usepackage{thmtools}

\definecolor{shadecolor}{gray}{0.90}
\declaretheoremstyle[
headfont=\normalfont\bfseries,
notefont=\mdseries, notebraces={(}{)},
bodyfont=\normalfont,
postheadspace=0.5em,
spaceabove=1pt,
mdframed={
  skipabove=8pt,
  skipbelow=8pt,
  hidealllines=true,
  backgroundcolor={shadecolor},
  innerleftmargin=4pt,
  innerrightmargin=4pt}
]{shaded}

\declaretheorem[within=section]{definition}
\declaretheorem[sibling=definition]{theorem}
\declaretheorem[sibling=definition]{proposition}
\declaretheorem[sibling=definition]{assumption}
\declaretheorem[sibling=definition]{corollary}

\declaretheorem[sibling=definition]{lemma}
\declaretheorem[sibling=definition]{remark}


%% file: sgd_shb_arxiv.bbl
\begin{thebibliography}{41}
\providecommand{\natexlab}[1]{#1}
\providecommand{\url}[1]{\texttt{#1}}
\expandafter\ifx\csname urlstyle\endcsname\relax
  \providecommand{\doi}[1]{doi: #1}\else
  \providecommand{\doi}{doi: \begingroup \urlstyle{rm}\Url}\fi

\bibitem[Attouch and Peypouquet(2016)]{Attouch16}
H{\'{e}}dy Attouch and Juan Peypouquet.
\newblock The rate of convergence of nesterov's accelerated forward-backward
  method is actually faster than 1/k\({}^{\mbox{2}}\).
\newblock \emph{{SIAM} Journal on Optimization}, 26\penalty0 (3):\penalty0
  1824--1834, 2016.

\bibitem[Bach and Moulines(2011)]{bach2011non}
Francis Bach and Eric Moulines.
\newblock Non-asymptotic analysis of stochastic approximation algorithms for
  machine learning.
\newblock In \emph{Neural Information Processing Systems (NIPS)}, 2011.

\bibitem[Bauschke and Combettes(2011)]{Bauschke_convex}
Heinz~H. Bauschke and Patrick~L. Combettes.
\newblock \emph{Convex Analysis and Monotone Operator Theory in Hilbert
  Spaces}.
\newblock Springer Publishing Company, Incorporated, 1st edition, 2011.

\bibitem[Beck and Teboulle(2009)]{beck2009fast}
Amir Beck and Marc Teboulle.
\newblock A fast iterative shrinkage-thresholding algorithm with application to
  wavelet-based image deblurring.
\newblock In \emph{2009 IEEE International Conference on Acoustics, Speech and
  Signal Processing}, pages 693--696. IEEE, 2009.

\bibitem[Bertsekas and Tsitsiklis(2000)]{bertsekas2000gradient}
Dimitri~P Bertsekas and John~N Tsitsiklis.
\newblock Gradient convergence in gradient methods with errors.
\newblock \emph{SIAM Journal on Optimization}, 10\penalty0 (3):\penalty0
  627--642, 2000.

\bibitem[Bottou(2003)]{Bottou03}
Leon Bottou.
\newblock Stochastic learning.
\newblock In \emph{Advanced Lectures on Machine Learning}, volume 3176, pages
  146--168, 2003.

\bibitem[Can et~al.(2019)Can, G{\"{u}}rb{\"{u}}zbalaban, and Zhu]{Can19}
Bugra Can, Mert G{\"{u}}rb{\"{u}}zbalaban, and Lingjiong Zhu.
\newblock Accelerated linear convergence of stochastic momentum methods in
  wasserstein distances.
\newblock In \emph{Proceedings of the 36th International Conference on Machine
  Learning}, pages 891--901, 2019.

\bibitem[Chambolle and Dossal(2015)]{Chambolle15}
Antonin Chambolle and Charles Dossal.
\newblock On the convergence of the iterates of the "fast iterative
  shrinkage/thresholding algorithm".
\newblock \emph{J. Optim. Theory Appl.}, 166\penalty0 (3):\penalty0 968--982,
  2015.

\bibitem[Defazio(2019)]{adefazio-curvedgeom2019}
Aaron Defazio.
\newblock On the curved geometry of accelerated optimization.
\newblock \emph{Advances in Neural Information Processing Systems 33 (NIPS
  2019)}, 2019.

\bibitem[Gadat et~al.(2018)Gadat, Panloup, and Saadane]{Gadat18}
Sébastien Gadat, Fabien Panloup, and Sofiane Saadane.
\newblock Stochastic heavy ball.
\newblock \emph{Electronic Journal of Statistics}, 12:\penalty0 461--529, 2018.

\bibitem[Gazagnadou et~al.(2019)Gazagnadou, Gower, and Salmon]{SAGAminib}
Nidham Gazagnadou, Robert~Mansel Gower, and Joseph Salmon.
\newblock Optimal mini-batch and step sizes for saga.
\newblock \emph{The International Conference on Machine Learning}, 2019.

\bibitem[Ghadimi et~al.(2015)Ghadimi, Feyzmahdavian, and
  Johansson]{Ghadimi2014}
Euhanna Ghadimi, Hamid Feyzmahdavian, and Mikael Johansson.
\newblock Global convergence of the heavy-ball method for convex optimization.
\newblock In \emph{2015 European Control Conference (ECC)}, pages 310--315,
  2015.

\bibitem[Ghadimi and Lan(2013)]{Ghadimi13}
Saeed Ghadimi and Guanghui Lan.
\newblock Stochastic first- and zeroth-order methods for nonconvex stochastic
  programming.
\newblock \emph{{SIAM} J. Optimization}, 23\penalty0 (4):\penalty0 2341--2368,
  2013.

\bibitem[Godichon-Baggioni(2016)]{godichon2016lp}
Antoine Godichon-Baggioni.
\newblock Lp and almost sure rates of convergence of averaged stochastic
  gradient algorithms with applications to online robust estimation.
\newblock \emph{arXiv preprint arXiv:1609.05479}, 2016.

\bibitem[Gower et~al.(2019)Gower, Loizou, Qian, Sailanbayev, Shulgin, and
  Richt{\'a}rik]{Gower19}
Robert~Mansel Gower, Nicolas Loizou, Xun Qian, Alibek Sailanbayev, Egor
  Shulgin, and Peter Richt{\'a}rik.
\newblock {SGD: General Analysis and Improved Rates}.
\newblock In \emph{Proceedings of the 36th International Conference on Machine
  Learning}, volume~97, pages 5200--5209, 2019.

\bibitem[Jain et~al.(2019)Jain, Nagaraj, and Netrapalli]{Jain19}
Prateek Jain, Dheeraj Nagaraj, and Praneeth Netrapalli.
\newblock Making the last iterate of sgd information theoretically optimal.
\newblock In \emph{Conference on Learning Theory, COLT 2019}, volume~99, pages
  1752--1755, 2019.

\bibitem[Khaled and Richt{\'{a}}rik(2020)]{Khaled20}
Ahmed Khaled and Peter Richt{\'{a}}rik.
\newblock Better theory for {SGD} in the nonconvex world.
\newblock \emph{arXiv:2002.03329}, 2020.

\bibitem[Kidambi et~al.(2018)Kidambi, Netrapalli, Jain, and Kakade]{Kidambi18}
Rahul Kidambi, Praneeth Netrapalli, Prateek Jain, and Sham~M. Kakade.
\newblock On the insufficiency of existing momentum schemes for stochastic
  optimization.
\newblock In \emph{International Conference on Learning Representations}, 2018.

\bibitem[Lan and Zhou(2017)]{lan2017}
Guanghui Lan and Yi~Zhou.
\newblock An optimal randomized incremental gradient method.
\newblock \emph{Mathematical programming}, pages 1--49, 2017.

\bibitem[Lee and Wright(2019)]{Lee19}
Ching-Pei Lee and Stephen Wright.
\newblock First-order algorithms converge faster than $o(1/k)$ on convex
  problems.
\newblock volume~97, pages 3754--3762, 2019.

\bibitem[Li and Orabona(2019)]{Li19}
Xiaoyu Li and Francesco Orabona.
\newblock On the convergence of stochastic gradient descent with adaptive
  stepsizes.
\newblock In \emph{The 22nd International Conference on Artificial Intelligence
  and Statistics, {AISTATS}}, pages 983--992, 2019.

\bibitem[Loizou and Richtárik(2018)]{Loizou2018}
Nicolas Loizou and Peter Richtárik.
\newblock {Momentum and Stochastic Momentum for Stochastic Gradient, Newton,
  Proximal Point and Subspace Descent Methods}.
\newblock \emph{arXiv:1712.09677}, 2018.

\bibitem[Loizou et~al.(2020)Loizou, Vaswani, Laradji, and
  Lacoste-Julien]{loizou2020stochastic}
Nicolas Loizou, Sharan Vaswani, Issam Laradji, and Simon Lacoste-Julien.
\newblock Stochastic polyak step-size for sgd: An adaptive learning rate for
  fast convergence.
\newblock \emph{arXiv preprint arXiv:2002.10542}, 2020.

\bibitem[Mertikopoulos et~al.(2020)Mertikopoulos, Hallak, Kavis, and
  Cevher]{mertikopoulos2020almost}
Panayotis Mertikopoulos, Nadav Hallak, Ali Kavis, and Volkan Cevher.
\newblock On the almost sure convergence of stochastic gradient descent in
  non-convex problems.
\newblock \emph{arXiv preprint arXiv:2006.11144}, 2020.

\bibitem[Nemirovski et~al.(2009)Nemirovski, Juditsky, Lan, and
  Shapiro]{Nemirovski09}
Arkadi Nemirovski, Anatoli~B. Juditsky, Guanghui Lan, and Alexander Shapiro.
\newblock Robust stochastic approximation approach to stochastic programming.
\newblock \emph{SIAM Journal on Optimization}, 19\penalty0 (4):\penalty0
  1574--1609, 2009.

\bibitem[Nesterov(2013)]{Nesterov-convex}
Yurii Nesterov.
\newblock \emph{Introductory lectures on convex optimization: A basic course},
  volume~87.
\newblock 2013.

\bibitem[Nguyen et~al.(2018)Nguyen, Nguyen, van Dijk, Richt{\'{a}}rik,
  Scheinberg, and Tak{\'{a}}c]{Nguyen18}
Lam~M. Nguyen, Phuong~Ha Nguyen, Marten van Dijk, Peter Richt{\'{a}}rik, Katya
  Scheinberg, and Martin Tak{\'{a}}c.
\newblock {SGD} and hogwild! convergence without the bounded gradients
  assumption.
\newblock In \emph{Proceedings of the 35th International Conference on Machine
  Learning}, pages 3747--3755, 2018.

\bibitem[Nocedal and Wright(2006)]{nocedal2006sequential}
Jorge Nocedal and Stephen~J Wright.
\newblock Sequential quadratic programming.
\newblock \emph{Numerical optimization}, pages 529--562, 2006.

\bibitem[Orabona(2020{\natexlab{a}})]{orabona2020almostsure}
Francesco Orabona.
\newblock Almost sure convergence of sgd on smooth non-convex functions.
\newblock
  \url{https://parameterfree.com/2020/10/05/almost-sure-convergence-of-sgd-on-smooth-non-convex-functions/},
  2020{\natexlab{a}}.
\newblock Accessed: 2021-01-28.

\bibitem[Orabona(2020{\natexlab{b}})]{orabona2020last}
Francesco Orabona.
\newblock Last iterate of sgd converges (even in bounded domains).
\newblock
  \url{https://parameterfree.com/2020/08/07/last-iterate-of-sgd-converges-even-in-unbounded-domains/},
  2020{\natexlab{b}}.
\newblock Accessed: 2021-01-28.

\bibitem[Orvieto et~al.(2019)Orvieto, K{\"{o}}hler, and Lucchi]{Orvieto19}
Antonio Orvieto, Jonas K{\"{o}}hler, and Aur{\'{e}}lien Lucchi.
\newblock The role of memory in stochastic optimization.
\newblock In \emph{Proceedings of the Thirty-Fifth Conference on Uncertainty in
  Artificial Intelligence, {UAI}}, page 128, 2019.

\bibitem[Polyak(1964)]{Polyak64}
B.~T. Polyak.
\newblock Some methods of speeding up the convergence of iteration methods.
\newblock \emph{USSR Computational Mathematics and Mathematical Physics},
  4:\penalty0 1--17, 1964.

\bibitem[Polyak(1987)]{polyak1987introduction}
BT~Polyak.
\newblock Introduction to optimization, translations series in mathematics and
  engineering.
\newblock \emph{Optimization Software}, 1987.

\bibitem[Robbins and Monro(1951)]{RobbinsMonro:1951}
H.~Robbins and S.~Monro.
\newblock A stochastic approximation method.
\newblock \emph{Annals of Mathematical Statistics}, 22:\penalty0 400--407,
  1951.

\bibitem[Robbins and Siegmund(1971)]{Robbins71}
Herbert Robbins and David Siegmund.
\newblock A convergence theorem for nonnegative almost supermartingales and
  some applications.
\newblock \emph{Optimizing methods in Statistics}, pages 233--257, 1971.

\bibitem[Sutskever et~al.(2013)Sutskever, Martens, Dahl, and
  Hinton]{Sutskever13}
Ilya Sutskever, James Martens, George Dahl, and Geoffrey Hinton.
\newblock On the importance of initialization and momentum in deep learning.
\newblock In \emph{Proceedings of the 30th International Conference on
  International Conference on Machine Learning - Volume 28}, page
  III–1139–III–1147, 2013.

\bibitem[Vaswani et~al.(2019{\natexlab{a}})Vaswani, Bach, and
  Schmidt]{Vaswani18}
Sharan Vaswani, Francis Bach, and Mark Schmidt.
\newblock Fast and faster convergence of {SGD} for over-parameterized models
  and an accelerated perceptron.
\newblock In \emph{The 22nd International Conference on Artificial Intelligence
  and Statistics, {AISTATS} 2019, 16-18 April 2019, Naha, Okinawa, Japan},
  volume~89, pages 1195--1204, 2019{\natexlab{a}}.

\bibitem[Vaswani et~al.(2019{\natexlab{b}})Vaswani, Mishkin, Laradji, Schmidt,
  Gidel, and Lacoste-Julien]{Vaswani19}
Sharan Vaswani, Aaron Mishkin, Issam Laradji, Mark Schmidt, Gauthier Gidel, and
  Simon Lacoste-Julien.
\newblock Painless stochastic gradient: Interpolation, line-search, and
  convergence rates.
\newblock In \emph{Advances in Neural Information Processing Systems 32}, pages
  3727--3740. 2019{\natexlab{b}}.

\bibitem[Vaswani et~al.(2020)Vaswani, Kunstner, Laradji, Meng, Schmidt, and
  Lacoste-Julien]{vaswani2020adaptive}
Sharan Vaswani, Frederik Kunstner, Issam Laradji, Si~Yi Meng, Mark Schmidt, and
  Simon Lacoste-Julien.
\newblock Adaptive gradient methods converge faster with over-parameterization
  (and you can do a line-search).
\newblock \emph{arXiv preprint arXiv:2006.06835}, 2020.

\bibitem[Yang et~al.(2016)Yang, Lin, and Li]{Yang16}
Tianbao Yang, Qihang Lin, and Zhe Li.
\newblock Unified convergence analysis of stochastic momentum methods for
  convex and non-convex optimization.
\newblock \emph{arXiv:1604.03257}, 2016.

\bibitem[Zhou et~al.(2017)Zhou, Mertikopoulos, Bambos, Boyd, and Glynn]{Zhou17}
Zhengyuan Zhou, Panayotis Mertikopoulos, Nicholas Bambos, Stephen~P. Boyd, and
  Peter~W. Glynn.
\newblock Stochastic mirror descent in variationally coherent optimization
  problems.
\newblock In \emph{Advances in Neural Information Processing Systems 30}, pages
  7040--7049, 2017.

\end{thebibliography}
